\documentclass[10pt,twocolumn,letterpaper]{article}

\PassOptionsToPackage{table}{xcolor}

\usepackage{etoc}
\usepackage{float}
\etocdepthtag.toc{main}
\usepackage[pagenumbers]{cvpr}

\usepackage[utf8]{inputenc}
\usepackage[T1]{fontenc}
\usepackage{amsthm}        
\usepackage{amsfonts}
\usepackage{bm}
\usepackage{nicefrac}
\usepackage{multirow}

\definecolor{ourblue}{RGB}{230,230,230}   

\makeatletter
\renewcommand{\paragraph}{\@startsection{paragraph}{4}{\z@}%
  {0.6ex \@plus 0.3ex \@minus .2ex}{-0.8em}{\normalfont\normalsize\bfseries}}
\makeatother

\usepackage[table]{xcolor}
\definecolor{ourblue}{HTML}{F2F2F2}
\usepackage{pifont}

\definecolor{cvprblue}{rgb}{0.21,0.49,0.74}
\usepackage[pagebackref,breaklinks,colorlinks,allcolors=cvprblue]{hyperref}
\usepackage{arydshln}
\def\confName{CVPR}
\def\confYear{2027}

\title{GeoFlowVLM: Geometry-Aware Joint Uncertainty for Frozen Vision-Language Embeddings}

\author{%
  Mayank Nautiyal$^{1,2}$ \quad Li Ju$^{1,2}$ \quad Tianru Zhang$^{1}$\\
  Andreas Hellander$^{1}$ \quad Ekta Vats$^{1}$ \quad Prashant Singh$^{1,2}$\\[4pt]
  $^{1}$Department of Information Technology, Uppsala University, Uppsala, Sweden\\
  $^{2}$SciLifeLab, Uppsala University, Uppsala, Sweden\\[2pt]
  {\small\texttt{\{mayank.nautiyal, li.ju, tianru.zhang, andreas.hellander, ekta.vats, prashant.singh\}@it.uu.se}}%
}

\begin{document}
\maketitle
\begin{abstract}
Frozen vision-language models (VLMs) typically represent images and text as deterministic embeddings, which do not directly provide estimates of aleatoric uncertainty arising from ambiguous image--text correspondences or epistemic uncertainty reflecting limited data support. We consider the post-hoc uncertainty quantification setting in which probabilistic adapters are trained on frozen VLM embeddings. Existing methods either fit unimodal distributions around individual embeddings, often in Euclidean space, or learn densities for each modality separately. Most address only one type of uncertainty, and to our knowledge, none learns a joint density of image and text embeddings. We propose \textit{GeoFlowVLM}, which learns a flexible joint density of image and text embeddings on the product of two hyperspheres using a single masked Riemannian flow matching model. The model provides both marginal and conditional densities, from which we derive two complementary uncertainty scores: an aleatoric score based on pointwise mutual information and an epistemic score based on joint density. In retrieval calibration across four datasets with CLIP and SigLIP backbones, GeoFlowVLM's aleatoric score is the only evaluated score for which Recall@1 decreases with increasing uncertainty in all 16 settings, particularly when a caption matches many images. Further, in selective zero-shot classification across five datasets, the epistemic score outperforms the geometry-aware baseline REPVLM, including under distribution shift.
\end{abstract}

\section{Introduction}\label{sec:intro}

\begin{figure}[t]
\centering
\includegraphics[width=\columnwidth]{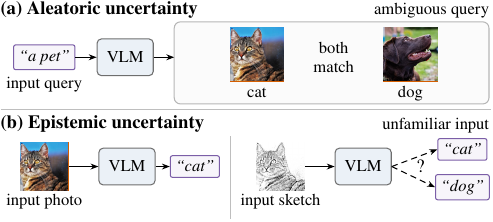}
\caption{\textbf{Two sources of uncertainty in a VLM.} (a) Aleatoric: the query ``a pet'' matches both images. (b) Epistemic: a VLM not trained on sketches is unsure on a sketch of the same cat.}
\label{fig:intro}
\end{figure}

Dual-encoder vision-language models (VLMs) such as CLIP~\citep{clip} and SigLIP~\citep{zhai2023sigmoid} are trained with contrastive objectives that project images and text onto a shared unit hypersphere via $\ell_2$-normalization, reducing cross-modal compatibility to cosine similarity.
This deterministic point representation is effective for retrieval and zero-shot transfer, but does not provide an estimate of prediction uncertainty: neither the \emph{aleatoric} ambiguity arising from the intrinsically one-to-many nature of vision-language correspondence~\citep{chun2021probabilistic}, nor the \emph{epistemic} uncertainty on inputs that lie far from the data seen during training (\cref{fig:intro}).
Quantifying both kinds of uncertainty without retraining or fine-tuning the encoders is the central challenge of post-hoc uncertainty quantification (UQ) in VLMs. The predictions therefore stay unchanged: retrieval rankings and zero-shot labels still come from cosine similarity, while uncertainty answers the separate question of when to trust them.

Existing post-hoc methods that train probabilistic models on frozen VLM embeddings (\cref{sec:related_works}) leave three gaps that no single approach addresses (\cref{tab:related}).
ProbVLM~\citep{probVLM_frozen_embeddings} and GroVE~\citep{venkataramanan2025probabilistic} use unimodal distributions in Euclidean space, assigning probability mass outside the unit hypersphere on which normalized embeddings lie.
AsymVLM~\citep{ju2025exploiting} respects this geometry but remains restricted to a unimodal parametric family.
REPVLM~\citep{ju2026epistemicuncertaintyquantificationpretrained} learns flexible hyperspherical densities but models image and text embeddings separately, without capturing their joint dependence.
Further, AsymVLM and GroVE address only aleatoric uncertainty, REPVLM only epistemic uncertainty, and ProbVLM estimates both through predicted embedding distributions and MC dropout. We observe that a single joint density $p(\mathbf{e}_I,\mathbf{e}_T)$ over the image and text embeddings provides a unified basis for quantifying both types of uncertainty.
Its marginal and conditional densities allow us to measure how much one modality constrains the other (aleatoric uncertainty), while the joint density captures how well an image--text pair is supported by the reference distribution (a proxy for epistemic uncertainty). This motivates learning a flexible joint density that respects the natural hyperspherical geometry of VLM embeddings.

We propose \textit{GeoFlowVLM}, which learns the joint density on the product of the two embedding hyperspheres with a single masked Riemannian flow matching model (\cref{fig:overview}). A block mask assigns each embedding one of three roles: \textit{generated} (transported by the flow), \textit{observed} (held fixed as a condition), or \textit{absent} (hidden from the model). At inference, changing the block mask provides two marginals and two conditionals; the joint density is recovered by combining a marginal with its corresponding conditional. All learned densities naturally support both sampling and log-density estimation, which we use to define the two uncertainty scores.

\paragraph{Contributions.}
\textbf{(i)} GeoFlowVLM, to our knowledge the \textit{first} post-hoc method to learn a flexible, geometry-aware joint density of frozen image and text embeddings as a unified basis for both types of uncertainty.

\noindent\textbf{(ii)} An aleatoric score derived from pointwise mutual information, computed from a conditional density and its corresponding marginal under the same model.
Over five seeds, it is the only evaluated score ranking retrieval errors consistently across two backbones, both retrieval directions, and all datasets, particularly when a caption fits many images.

\noindent\textbf{(iii)} An epistemic score based on the joint density of an image embedding and its predicted class embedding, capturing their support in the reference distribution.
Across five datasets and two backbones, with five seeds each, it improves selective prediction over REPVLM in all ten settings and gives the highest selective accuracy in seven, including all four under distribution shift.

\section{Related Work}\label{sec:related_works}

\begin{figure*}[t]
\centering
\includegraphics[width=\textwidth]{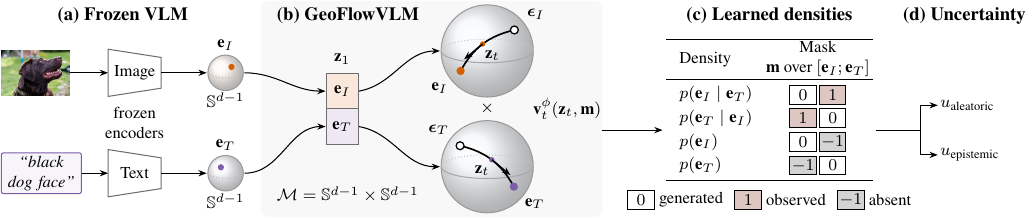}
\caption{\textbf{GeoFlowVLM overview.} (a) A frozen VLM embeds an image--text pair. (b) One masked velocity model $\mathbf{v}_t^{\phi}(\mathbf{z}_t, \mathbf{m})$ learns their joint distribution by Riemannian flow matching on $\mathcal{M}$. (c) Each mask selects a density. (d) Both uncertainties are read off these densities.}

\label{fig:overview}
\end{figure*}

\paragraph{Uncertainty quantification for VLMs.}
Prior work quantifies VLM uncertainty by modeling distributions over image and text embeddings. These methods fall into two broad groups. The first group trains the encoders themselves or requires access to their internal layers. PCME~\citep{chun2021probabilistic}, PCME++~\citep{chun2024pcmepp}, and ProLIP~\citep{chun2025prolip} learn embedding distributions while training the encoders, while BayesVLM~\citep{baumann2024post} derives them from a Laplace approximation over the projection layers of both encoders. The second group trains probabilistic models using only the output embeddings of frozen encoders~\citep{probVLM_frozen_embeddings,venkataramanan2025probabilistic,ju2025exploiting,ju2026epistemicuncertaintyquantificationpretrained}. Several training-free methods also exist, targeting specific tasks such as zero-shot OOD detection and retrieval uncertainty~\citep{ming2022delving,gomez2025uai-top1}. In this work, we focus on the \emph{second group}, which allows us to learn embedding distributions without accessing the encoders' internal layers or modifying their parameters.
We use the probabilistic component to flag potentially unreliable predictions while leaving the deterministic predictor's retrieval rankings and zero-shot labels unchanged.

Within our setting, existing methods differ in their distributional assumptions and the types of uncertainty they estimate (\cref{tab:related}).
ProbVLM~\citep{probVLM_frozen_embeddings} estimates aleatoric uncertainty through a generalized Gaussian distribution (GGD) around each embedding and epistemic uncertainty through MC dropout~\citep{gal2016dropout}.
GroVE~\citep{venkataramanan2025probabilistic} and AsymVLM~\citep{ju2025exploiting} focus only on aleatoric uncertainty, using Gaussian process (GP) predictive distributions and von Mises--Fisher (vMF) or power spherical (PS) distributions, respectively.
REPVLM~\citep{ju2026epistemicuncertaintyquantificationpretrained} instead focuses on epistemic uncertainty, learning a separate marginal density for each modality using flows on the hypersphere.
Unlike these methods, GeoFlowVLM learns one flexible \emph{joint} density of image and text embeddings on the product of their hyperspheres and derives both types of uncertainty from it.

\paragraph{Joint distribution modeling.}
Learning joint distributions with expressive generative models has been explored in other fields, including simulation-based inference~\citep{gloeckler2024all,nautiyal2026oneflowsbi}.
In vision--language modeling, related work has focused on generation.
UniDiffuser~\citep{bao2023one} assigns each modality its own diffusion time, allowing one transformer to sample images, texts or image-text pairs, while OmniFlow~\citep{li2025omniflow} and MUNI~\citep{muni2026} pursue any-to-any generation across image, text and audio.
Closer to our setting, unCLIP~\citep{ramesh2022hierarchical} and Corgi~\citep{zhou2023shifted} model CLIP image embeddings directly, but only their conditional distribution given text, which serves as a prior for an image decoder. These vision--language models target generation in Euclidean latent spaces, whereas our goal is to assess the reliability of a frozen VLM's predictions using its normalized embeddings, which lie on unit hyperspheres.

\begin{table}[t]
\centering
\footnotesize
\setlength{\tabcolsep}{4pt}
\renewcommand{\arraystretch}{1.0}
\newcommand{\no}{\textcolor{gray!60}{--}}
\caption{Post-hoc UQ methods that train probabilistic models on frozen VLM embeddings. $^\dagger$Trained contrastively, not by maximum likelihood.}
\label{tab:related}
\begin{tabular}{@{}lcccccc@{}}
\toprule
& \multicolumn{2}{c}{\textbf{Uncertainty}} & \multicolumn{2}{c}{\textbf{Model}} & \multicolumn{2}{c}{\textbf{Density}} \\
\cmidrule(lr){2-3}\cmidrule(lr){4-5}\cmidrule(l){6-7}
\textbf{Method} & Alea. & Epis. & Family & Space & Marg. & Cond. \\
\midrule
\multicolumn{7}{@{}l}{\emph{Distribution around each embedding}} \\
ProbVLM  & \checkmark & \checkmark & GGD    & $\mathbb{R}^d$     & \no & \no \\
GroVE    & \checkmark & \no        & GP     & $\mathbb{R}^d$     & \no & \no \\
AsymVLM  & \checkmark & \no        & vMF/PS & $\mathbb{S}^{d-1}$ & \no & \checkmark$^\dagger$ \\
\addlinespace[1pt]
\multicolumn{7}{@{}l}{\emph{Marginal density per modality}} \\
REPVLM   & \no & \checkmark & Flow & $\mathbb{S}^{d-1}$ & \checkmark & \no \\
\addlinespace[1pt]
\multicolumn{7}{@{}l}{\emph{Joint density}} \\
\rowcolor{ourblue}
\textbf{GeoFlowVLM} & \checkmark & \checkmark & Flow & $(\mathbb{S}^{d-1})^2$ & \checkmark & \checkmark \\
\bottomrule
\end{tabular}
\vspace{-3mm}
\end{table}

\section{GeoFlowVLM}\label{sec:method}

\paragraph{Problem setup.}
Let $\mathbf{e}_I, \mathbf{e}_T \in \mathbb{S}^{d-1}$ denote the unit-normalized embeddings of an image--text pair $(I,T)$ produced by a frozen contrastive VLM.
The embedding pair lies on the product hypersphere $\mathcal{M} = \mathbb{S}^{d-1} \times \mathbb{S}^{d-1}$ (Fig.~\ref{fig:overview}a).
The joint density admits two equivalent factorizations,
\begin{equation}
p(\mathbf{e}_I, \mathbf{e}_T) = p(\mathbf{e}_T)\, p(\mathbf{e}_I \mid \mathbf{e}_T) = p(\mathbf{e}_I)\, p(\mathbf{e}_T \mid \mathbf{e}_I),
\label{eq:factorization}
\end{equation}
where $p(\mathbf{e}_I)$ and $p(\mathbf{e}_T)$ are the image and text marginals, and $p(\mathbf{e}_I \mid \mathbf{e}_T)$ and $p(\mathbf{e}_T \mid \mathbf{e}_I)$ are the corresponding cross-modal conditional densities. GeoFlowVLM models the joint \cref{eq:factorization} through both factorizations: one model, trained by masked Riemannian flow matching~\cite{chen2024flow} on $\mathcal{M}$, learns both marginals and both conditionals (Fig.~\ref{fig:overview}b,c).

\paragraph{Selecting a density with a mask.}
The four densities differ only in the role of each embedding, which is either \textit{generated} (produced by the flow), \textit{observed} (held fixed as a condition), or \textit{absent} (excluded from the flow).
We represent a point on $\mathcal{M}$ as a joint state $\mathbf{z}$ formed by concatenating an image block with a text block, and a data pair as $\mathbf{z}_1=[\mathbf{e}_I;\mathbf{e}_T]$.
We encode the roles with a block-constant mask $\mathbf{m}\in\{-1,0,1\}^{2d}$ that takes the value $0$ on a generated block, $1$ on an observed block and $-1$ on an absent block.
Each density corresponds to one mask (\cref{fig:overview}c): $p(\mathbf{e}_I\mid\mathbf{e}_T)$ generates the image block and observes the text block, and $p(\mathbf{e}_I)$ generates the image block and marks the text block absent.
For notational convenience, we derive from $\mathbf{m}$ the binary masks $\mathbf{m}^{\mathrm{obs}},\mathbf{m}^{\mathrm{gen}}\in\{0,1\}^{2d}$, which equal one on the observed and on the generated coordinates, respectively, and are both zero on absent coordinates.

\paragraph{Masked geodesic path.}
The flow evolves over $t\in[0,1]$ from a base sample $\mathbf{z}_0=[\boldsymbol{\epsilon}_I;\boldsymbol{\epsilon}_T]\sim\mathcal{U}(\mathcal{M})$ at $t=0$, where $\mathcal{U}(\mathcal{M})$ is the uniform distribution on $\mathcal{M}$, to a data pair $\mathbf{z}_1=[\mathbf{e}_I;\mathbf{e}_T]\sim p(\mathbf{e}_I,\mathbf{e}_T)$ at $t=1$ (\cref{fig:overview}b).
We connect the two with the masked geodesic interpolant,
\begin{equation}
\mathbf{z}_t=\mathbf{m}^{\mathrm{obs}}\odot\mathbf{z}_1+\mathbf{m}^{\mathrm{gen}}\odot\exp_{\mathbf{z}_0}\!\big(t\log_{\mathbf{z}_0}\mathbf{z}_1\big),
\label{eq:masked_interpolant}
\end{equation}
where $\odot$ denotes elementwise multiplication, and $\exp$ and $\log$ are the exponential and logarithmic maps on each sphere of $\mathcal{M}$~\citep{chen2024flow}.
Along this interpolant, an observed block keeps its embedding, an absent block stays at zero, and the generated block moves along a geodesic from noise at $t=0$ to its embedding at $t=1$.
This geodesic has a closed-form velocity,
\begin{equation}
\mathbf{v}_t(\mathbf{z}_t\mid\mathbf{z}_1,\mathbf{m})=\mathbf{m}^{\mathrm{gen}}\odot\frac{\log_{\mathbf{z}_t}\mathbf{z}_1}{1-t},
\label{eq:target_velocity}
\end{equation}
where the logarithm is evaluated on the generated block only.

\paragraph{Training objective.}
The velocity model $\mathbf{v}_t^{\phi}(\mathbf{z}_t,\mathbf{m})$ takes the mask as input and is trained to match the target velocity on the generated coordinates.
Its output is projected onto the tangent space of each sphere, $\mathrm{Proj}_{\mathbf{z}}(\mathbf{v})=\mathbf{v}-\langle\mathbf{v},\mathbf{z}\rangle\,\mathbf{z}$ applied to the image and text blocks separately. We define the masked velocity residual,
\begin{equation}
\boldsymbol{r}_t^{\phi}=\mathbf{m}^{\mathrm{gen}}\odot\mathrm{Proj}_{\mathbf{z}_t}\big(\mathbf{v}_t^{\phi}(\mathbf{z}_t,\mathbf{m})\big)-\mathbf{v}_t(\mathbf{z}_t\mid\mathbf{z}_1,\mathbf{m}),
\label{eq:residual}
\end{equation}
and GeoFlowVLM minimizes,
\begin{equation}
\mathcal{L}_{\mathrm{GeoFlowVLM}}(\phi)=\mathbb{E}_{\substack{t\sim\mathcal{U}[0,1],\ \mathbf{m}\sim p(\mathbf{m})\\ \mathbf{z}_0\sim\mathcal{U}(\mathcal{M})\\ \mathbf{z}_1\sim p(\mathbf{e}_I,\mathbf{e}_T)}}\Big[\tfrac{1}{d}\big\|\boldsymbol{r}_t^{\phi}\big\|^2\Big],
\label{eq:loss}
\end{equation}
where $p(\mathbf{m})$ is the uniform distribution over the four masks of \cref{fig:overview}c, so that the four densities receive equal weight, and the factor $1/d$ averages over the $d$ generated coordinates. Supp.~\ref{sec:supp_sampling_density} describes sampling and log-density evaluation with the trained model.

\paragraph{Consistency at the optimum.}
We show that, with sufficient model capacity, the population optimum of \cref{eq:loss} recovers the marginals and conditionals of the joint distribution of paired image and text embeddings.
For a fixed mask, \cref{eq:loss} is a conditional flow matching objective on the generated block, so under the standard flow matching assumptions~\citep{lipman2023flow,chen2024flow}, its population minimizer $\mathbf{v}_t^{\star}$ is the conditional mean of the target velocity,
\begin{equation}
\mathbf{m}^{\mathrm{gen}}\odot\mathbf{v}_t^{\star}(\mathbf{z},\mathbf{m})=\mathbb{E}\big[\mathbf{v}_t(\mathbf{z}_t\mid\mathbf{z}_1,\mathbf{m})\mid\mathbf{z}_t=\mathbf{z},\mathbf{m}\big].
\label{eq:optimum}
\end{equation}
Since $\mathbf{z}_t$ carries the observed block by \cref{eq:masked_interpolant}, the expectation also conditions on the observed embedding.
This conditional mean is the velocity field that generates the probability path induced by the masked interpolant.
At this optimum, the flow of $\mathbf{v}_t^{\star}$ transports the uniform base distribution on the generated block to the target marginal or conditional distribution selected by the mask.

The full objective in \cref{eq:loss} averages four such terms, one per mask.
Because the mask is an input, the terms are evaluated on disjoint inputs in the mask-augmented space, allowing a model with sufficient capacity to attain all four optima simultaneously.
At that optimum, the recovered densities are Bayes-compatible: both factorizations in \cref{eq:factorization} recover the same joint distribution.
This compatibility is a property of the optimum, not of the parameterization: parameter sharing does not enforce it, so finite-sample, approximation and optimization errors can make the two learned factorizations disagree.
This motivates the velocity parameterization below.

\paragraph{Velocity model.}
A conditional such as $p(\mathbf{e}_I\mid\mathbf{e}_T)$ differs from its marginal $p(\mathbf{e}_I)$ only by what the observed embedding $\mathbf{e}_T$ adds.
Inspired by the additive score decomposition of classifier guidance~\citep{dhariwal2021diffusion}, we therefore write each conditional velocity as the marginal velocity plus a learned correction, $\mathbf{v}_t^{\phi}=\mathbf{v}_t^{\theta}+\boldsymbol{\Delta}_t^{\psi}$ with $\phi=(\theta,\psi)$, and each marginal velocity as $\mathbf{v}_t^{\theta}$ alone.
The marginal branch $\mathbf{v}_t^{\theta}$ sees only the generated block and is updated only by the marginal terms of \cref{eq:loss}; the correction branch $\boldsymbol{\Delta}_t^{\psi}$ also sees the observed block, starts from a zero-initialized output as in ControlNet~\citep{zhang2023adding}, and is updated by the conditional terms (Supp.~\ref{sec:supp_architecture}).
This split is motivated by a consistency check. The two factorizations of \cref{eq:factorization} describe one joint, so the expected gain in log-density from conditioning the image on the text and the gain from conditioning the text on the image should both equal the mutual information. On held-out image--text pairs, we observe that the two expected log-density gains disagree for a single branch trained on all four densities, even at matched parameter count, whereas the split yields agreement under this diagnostic (Supp.~\cref{tab:supp_mi_check}).


\section{Uncertainty Quantification}\label{sec:uq}
In this section, we derive two uncertainty scores from the four densities learned by GeoFlowVLM: an aleatoric score (\cref{sec:uq:aleatoric}) and an epistemic score (\cref{sec:uq:epistemic}).

\begin{figure*}[t]
\centering
\includegraphics[width=\textwidth]{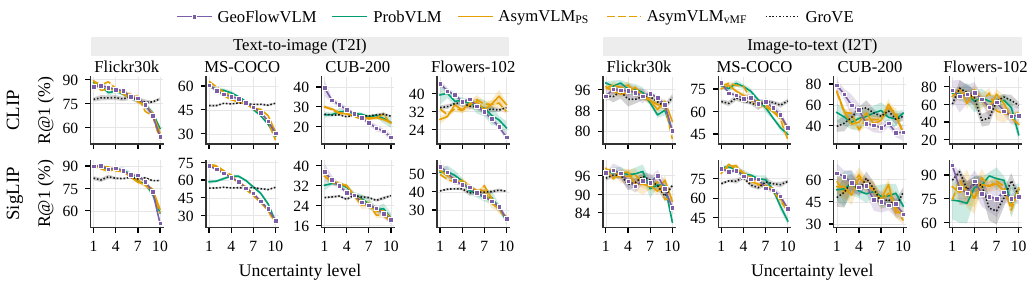}
\caption{Recall@1 per uncertainty bin, from least to most uncertain, for CLIP (top) and SigLIP (bottom). Mean $\pm$ std over five seeds.}
\label{fig:calibration}
\end{figure*}


\subsection{Aleatoric Uncertainty}\label{sec:uq:aleatoric}
In VLMs, aleatoric uncertainty arises naturally in cross-modal retrieval, where a query may correctly match multiple items in the other modality. This ambiguity is inherent in the data and irreducible given the query~\citep{hullermeier2021aleatoric}, but it can be quantified. We derive a score for a text query $\mathbf{e}_T$; the image-query case is symmetric.

\paragraph{Ranking by PMI.}
For a text query $\mathbf{e}_T$, let $\mathcal{C}$ be a set of candidate images that contains one true match, drawn from $p(\mathbf{e}_I\mid\mathbf{e}_T)$, while every other image in $\mathcal{C}$ is drawn i.i.d.\ from the population $p(\mathbf{e}_I)$. If each image in $\mathcal{C}$ is a priori equally likely to be the match, the posterior probability that a candidate $\mathbf{e}_I\in\mathcal{C}$ is the true match for the text query $\mathbf{e}_T$ is proportional to the density ratio $p(\mathbf{e}_I\mid\mathbf{e}_T)/p(\mathbf{e}_I)$~\citep{oord2018representation}.
Ranking the candidates by this ratio is therefore Bayes optimal~\citep{cover2006elements}, and its logarithm is the pointwise mutual information (PMI) of the two embeddings~\citep{church1990word},
\begin{equation}
\mathrm{PMI}(\mathbf{e}_T,\mathbf{e}_I)=\log\frac{p(\mathbf{e}_I\mid\mathbf{e}_T)}{p(\mathbf{e}_I)}.
\label{eq:pmi}
\end{equation}
The PMI measures how conditioning on a text changes an image's log density.
Since both $p(\mathbf{e}_I)$ and $p(\mathbf{e}_I\mid\mathbf{e}_T)$ integrate to one over $\mathbb{S}^{d-1}$, any density increase in one region must be offset by a decrease elsewhere.
A specific text concentrates its probability mass on a narrow set of compatible images, where the conditional density far exceeds the marginal and the PMI is large and positive, while an ambiguous text spreads its mass across more compatible images, bringing the density ratio closer to one.

\paragraph{Score.}
A query is thus intrinsically ambiguous when its true matches have an expected PMI close to zero, since a ranker cannot reliably distinguish them from population distractors.
GeoFlowVLM quantifies this ambiguity for a text query $\mathbf{e}_T$ as the negative expected PMI of its compatible images, so larger values indicate greater uncertainty,
\begin{equation}
\begin{aligned}
u_{\mathrm{aleatoric}}(\mathbf{e}_T)&=-\mathbb{E}_{\mathbf{e}_I\sim p(\mathbf{e}_I\mid\mathbf{e}_T)}\big[\mathrm{PMI}(\mathbf{e}_T,\mathbf{e}_I)\big]\\
&=-\mathbb{E}_{\mathbf{e}_I\sim p(\mathbf{e}_I\mid\mathbf{e}_T)}\big[\log p(\mathbf{e}_I\mid\mathbf{e}_T)-\log p(\mathbf{e}_I)\big]\\
&=-D_{\mathrm{KL}}\big(p(\mathbf{e}_I\mid\mathbf{e}_T)\,\|\,p(\mathbf{e}_I)\big).
\end{aligned}
\label{eq:info_score}
\end{equation}

\paragraph{Interpretation.}
\textit{Uninformative text:} if the text does not match the image, then $p(\mathbf{e}_I\mid\mathbf{e}_T)=p(\mathbf{e}_I)$, every PMI is zero, and $u_{\mathrm{aleatoric}}=0$; since the KL divergence is non-negative, this is the maximum of the score, so an uninformative text is maximally uncertain.
\textit{Ambiguous text:} a text whose possible matches cover a large part of the image population gives each a small PMI, so $u_{\mathrm{aleatoric}}$ stays close to zero. \textit{Specific text:} a text whose possible matches cover only a small part of the image population gives them a large PMI, so $u_{\mathrm{aleatoric}}$ is strongly negative, and it decreases without bound as $p(\mathbf{e}_I\mid\mathbf{e}_T)$ concentrates on a single image. Supp.~\ref{sec:supp_qualitative_aleatoric} illustrates these cases on real captions.

\paragraph{Properties.}
First, the score does not place any parametric assumption on the conditional: GeoFlowVLM models it with a flow, which naturally captures the multimodality that arises when a query has several distinct matches.
Second, it measures the conditional against the image population rather than a uniform reference. This matters because CLIP embeddings concentrate in narrow regions of the hypersphere~\citep{liang2022mind}, so even a fully ambiguous query has a conditional that is far from uniform.

\paragraph{Estimation.}
We estimate \cref{eq:info_score} by drawing $M$ samples from the conditional flow $p(\mathbf{e}_I\mid\mathbf{e}_T)$ and taking the negative mean of their PMI.
Sampling and the conditional log-density use the mask $\mathbf{m}=(0,1)$, and the marginal log-density uses $\mathbf{m}=(0,-1)$, so each query costs $M$ samples and $2M$ log-density evaluations.




\subsection{Epistemic Uncertainty}\label{sec:uq:epistemic}
In a VLM, epistemic uncertainty arises when an input image/text, differs from the data the model was trained on~\citep{hullermeier2021aleatoric}. Since we consider a frozen VLM and have access only to its deterministic embeddings, we use the joint density learned by GeoFlowVLM as a surrogate for epistemic uncertainty.

\paragraph{Density as a surrogate.}
This choice is motivated by an analogy with local estimation.
Consider a predictor that estimates its output at an input $x$ from nearby training samples, as kernel regression does.
For a training set of size $n$ and a fixed neighborhood size, the local sample count is roughly proportional to $n\,p(x)$.
With the local output variance held fixed, the estimation variance therefore scales as $1/(n\,p(x))$~\citep{kotelevskii2022nonparametric}.
This links low density to limited local data support and the estimation uncertainty that additional examples can reduce.
Deep Deterministic Uncertainty (DDU) uses feature log density as a surrogate for epistemic uncertainty in a deterministic classifier~\citep{mukhoti2023deep}, and REPVLM applies the same reasoning to frozen VLM embeddings~\citep{ju2026epistemicuncertaintyquantificationpretrained}.
Our score uses the densities learned in \cref{sec:method}, so it measures novelty relative to the data they were fitted to.

\paragraph{Score.}
We derive a score to quantify the epistemic uncertainty of a frozen VLM in zero-shot classification.
Let $\mathbf{e}_T^{(1)}$ denote the embedding of the class that the VLM ranks highest for an image $\mathbf{e}_I$.
The prediction then claims that the image and this class embedding belong together.
We therefore score the prediction by the negative log density of this pair under the training distribution, which the chain rule splits into two terms,
\begin{equation}
\begin{aligned}
u_{\mathrm{epistemic}}(\mathbf{e}_I)&=-\log p\big(\mathbf{e}_I,\mathbf{e}_T^{(1)}\big)\\
&=-\log p(\mathbf{e}_I)-\log p\big(\mathbf{e}_T^{(1)}\mid\mathbf{e}_I\big).
\end{aligned}
\label{eq:epi_score}
\end{equation}
The first term captures how unusual an image is under the training distribution, while the second captures how rarely such an image is paired with the predicted class.
A similar decomposition is used in Posterior Network~\citep{charpentier2020posterior}.
GeoFlowVLM computes these terms using the masks $\mathbf{m}=(0,-1)$ and $\mathbf{m}=(1,0)$, respectively, which requires two log-density evaluations per image.

\paragraph{Interpretation.}
\textit{Novel image:} if few training images lie near the input image, $-\log p(\mathbf{e}_I)$ is large, which raises the score whatever the predicted class.
\textit{Familiar image, unlikely class:} if the image is familiar but the training data rarely pair such images with the predicted class, $-\log p(\mathbf{e}_T^{(1)}\mid\mathbf{e}_I)$ is large, which again raises the score.
\textit{Familiar image, likely class:} if the image is familiar and the training data often pair such images with the predicted class, both terms are small and the score is low.
Supp.~\ref{sec:supp_qualitative_epistemic} illustrates these cases on real images.

\paragraph{Relation to REPVLM.}
The selective prediction score of REPVLM corresponds to the independence case of \cref{eq:epi_score}: it sums the negative log marginal densities of the image and its predicted class embedding, $u_{\mathrm{REPVLM}}(\mathbf{e}_I)=-\log p(\mathbf{e}_I)-\log p\big(\mathbf{e}_T^{(1)}\big)$.
With the same marginals, the two scores differ by $\delta(\mathbf{e}_I)=u_{\mathrm{REPVLM}}(\mathbf{e}_I)-u_{\mathrm{epistemic}}(\mathbf{e}_I)$, the PMI of \cref{eq:pmi} in its text-side form,
\begin{equation}
\delta(\mathbf{e}_I)=\log p\big(\mathbf{e}_T^{(1)}\mid\mathbf{e}_I\big)-\log p\big(\mathbf{e}_T^{(1)}\big).
\label{eq:pmi_gap}
\end{equation}
The score $u_{\mathrm{epistemic}}$ therefore accounts for the dependence between the image and its predicted class embedding, which the product of marginals omits.

\begin{table*}[t]
\centering
\caption{Spearman $\rho$ and $R^2$ between uncertainty bin and Recall@1, mean (std below) over five seeds. Bold: best, underline: second best.}
\label{tab:uncertainty_calibration}
\scriptsize
\setlength{\tabcolsep}{2.4pt}
\renewcommand{\arraystretch}{0.85}
\begin{tabular}{@{}ll cc cc cc cc @{\hspace{4pt}}c@{\hspace{4pt}} cc cc cc cc}
\toprule
& & \multicolumn{8}{c}{\textbf{CLIP ViT-H/14}} & & \multicolumn{8}{c}{\textbf{SigLIP ViT-L/16-384}} \\
\cmidrule(lr){3-10}\cmidrule(lr){12-19}
& & \multicolumn{2}{c}{Flickr30k} & \multicolumn{2}{c}{MS-COCO} & \multicolumn{2}{c}{CUB-200} & \multicolumn{2}{c}{Flowers-102} & & \multicolumn{2}{c}{Flickr30k} & \multicolumn{2}{c}{MS-COCO} & \multicolumn{2}{c}{CUB-200} & \multicolumn{2}{c}{Flowers-102} \\
\cmidrule(lr){3-4}\cmidrule(lr){5-6}\cmidrule(lr){7-8}\cmidrule(lr){9-10}\cmidrule(lr){12-13}\cmidrule(lr){14-15}\cmidrule(lr){16-17}\cmidrule(lr){18-19}
& \textbf{Method} & $\rho\downarrow$ & $R^{2}\uparrow$ & $\rho\downarrow$ & $R^{2}\uparrow$ & $\rho\downarrow$ & $R^{2}\uparrow$ & $\rho\downarrow$ & $R^{2}\uparrow$ & & $\rho\downarrow$ & $R^{2}\uparrow$ & $\rho\downarrow$ & $R^{2}\uparrow$ & $\rho\downarrow$ & $R^{2}\uparrow$ & $\rho\downarrow$ & $R^{2}\uparrow$ \\
\midrule
\multirow{8.88}{*}{\rotatebox[origin=c]{90}{\textbf{I2T}}}
  & ProbVLM
    & \begin{tabular}[c]{@{}c@{}}$\mathbf{-.862}$\\[-0.2pt]{\textcolor{black!55}{$\scriptscriptstyle\pm.057$}}\end{tabular} & \begin{tabular}[c]{@{}c@{}}$\mathbf{.798}$\\[-0.2pt]{\textcolor{black!55}{$\scriptscriptstyle\pm.035$}}\end{tabular} & \begin{tabular}[c]{@{}c@{}}$-.925$\\[-0.2pt]{\textcolor{black!55}{$\scriptscriptstyle\pm.038$}}\end{tabular} & \begin{tabular}[c]{@{}c@{}}$.863$\\[-0.2pt]{\textcolor{black!55}{$\scriptscriptstyle\pm.039$}}\end{tabular} & \begin{tabular}[c]{@{}c@{}}$+.110$\\[-0.2pt]{\textcolor{black!55}{$\scriptscriptstyle\pm.120$}}\end{tabular} & \begin{tabular}[c]{@{}c@{}}$.017$\\[-0.2pt]{\textcolor{black!55}{$\scriptscriptstyle\pm.026$}}\end{tabular} & \begin{tabular}[c]{@{}c@{}}$-.622$\\[-0.2pt]{\textcolor{black!55}{$\scriptscriptstyle\pm.234$}}\end{tabular} & \begin{tabular}[c]{@{}c@{}}$.471$\\[-0.2pt]{\textcolor{black!55}{$\scriptscriptstyle\pm.164$}}\end{tabular}
    & 
    & \begin{tabular}[c]{@{}c@{}}$-.617$\\[-0.2pt]{\textcolor{black!55}{$\scriptscriptstyle\pm.041$}}\end{tabular} & \begin{tabular}[c]{@{}c@{}}$.519$\\[-0.2pt]{\textcolor{black!55}{$\scriptscriptstyle\pm.042$}}\end{tabular} & \begin{tabular}[c]{@{}c@{}}$\underline{-.954}$\\[-0.2pt]{\textcolor{black!55}{$\scriptscriptstyle\pm.026$}}\end{tabular} & \begin{tabular}[c]{@{}c@{}}$.800$\\[-0.2pt]{\textcolor{black!55}{$\scriptscriptstyle\pm.057$}}\end{tabular} & \begin{tabular}[c]{@{}c@{}}$-.413$\\[-0.2pt]{\textcolor{black!55}{$\scriptscriptstyle\pm.151$}}\end{tabular} & \begin{tabular}[c]{@{}c@{}}$.196$\\[-0.2pt]{\textcolor{black!55}{$\scriptscriptstyle\pm.158$}}\end{tabular} & \begin{tabular}[c]{@{}c@{}}$+.229$\\[-0.2pt]{\textcolor{black!55}{$\scriptscriptstyle\pm.164$}}\end{tabular} & \begin{tabular}[c]{@{}c@{}}$.080$\\[-0.2pt]{\textcolor{black!55}{$\scriptscriptstyle\pm.054$}}\end{tabular} \\
  & AsymVLM$_{\mathrm{PS}}$
    & \begin{tabular}[c]{@{}c@{}}$-.799$\\[-0.2pt]{\textcolor{black!55}{$\scriptscriptstyle\pm.060$}}\end{tabular} & \begin{tabular}[c]{@{}c@{}}$\underline{.670}$\\[-0.2pt]{\textcolor{black!55}{$\scriptscriptstyle\pm.060$}}\end{tabular} & \begin{tabular}[c]{@{}c@{}}$\mathbf{-.971}$\\[-0.2pt]{\textcolor{black!55}{$\scriptscriptstyle\pm.014$}}\end{tabular} & \begin{tabular}[c]{@{}c@{}}$\mathbf{.896}$\\[-0.2pt]{\textcolor{black!55}{$\scriptscriptstyle\pm.019$}}\end{tabular} & \begin{tabular}[c]{@{}c@{}}$\underline{-.396}$\\[-0.2pt]{\textcolor{black!55}{$\scriptscriptstyle\pm.102$}}\end{tabular} & \begin{tabular}[c]{@{}c@{}}$.220$\\[-0.2pt]{\textcolor{black!55}{$\scriptscriptstyle\pm.047$}}\end{tabular} & \begin{tabular}[c]{@{}c@{}}$-.632$\\[-0.2pt]{\textcolor{black!55}{$\scriptscriptstyle\pm.082$}}\end{tabular} & \begin{tabular}[c]{@{}c@{}}$.494$\\[-0.2pt]{\textcolor{black!55}{$\scriptscriptstyle\pm.122$}}\end{tabular}
    & 
    & \begin{tabular}[c]{@{}c@{}}$\underline{-.653}$\\[-0.2pt]{\textcolor{black!55}{$\scriptscriptstyle\pm.118$}}\end{tabular} & \begin{tabular}[c]{@{}c@{}}$\underline{.556}$\\[-0.2pt]{\textcolor{black!55}{$\scriptscriptstyle\pm.105$}}\end{tabular} & \begin{tabular}[c]{@{}c@{}}$-.925$\\[-0.2pt]{\textcolor{black!55}{$\scriptscriptstyle\pm.025$}}\end{tabular} & \begin{tabular}[c]{@{}c@{}}$.856$\\[-0.2pt]{\textcolor{black!55}{$\scriptscriptstyle\pm.014$}}\end{tabular} & \begin{tabular}[c]{@{}c@{}}$-.564$\\[-0.2pt]{\textcolor{black!55}{$\scriptscriptstyle\pm.162$}}\end{tabular} & \begin{tabular}[c]{@{}c@{}}$.379$\\[-0.2pt]{\textcolor{black!55}{$\scriptscriptstyle\pm.209$}}\end{tabular} & \begin{tabular}[c]{@{}c@{}}$-.122$\\[-0.2pt]{\textcolor{black!55}{$\scriptscriptstyle\pm.081$}}\end{tabular} & \begin{tabular}[c]{@{}c@{}}$.024$\\[-0.2pt]{\textcolor{black!55}{$\scriptscriptstyle\pm.015$}}\end{tabular} \\
  & AsymVLM$_{\mathrm{vMF}}$
    & \begin{tabular}[c]{@{}c@{}}$\underline{-.818}$\\[-0.2pt]{\textcolor{black!55}{$\scriptscriptstyle\pm.033$}}\end{tabular} & \begin{tabular}[c]{@{}c@{}}$.655$\\[-0.2pt]{\textcolor{black!55}{$\scriptscriptstyle\pm.047$}}\end{tabular} & \begin{tabular}[c]{@{}c@{}}$\underline{-.942}$\\[-0.2pt]{\textcolor{black!55}{$\scriptscriptstyle\pm.031$}}\end{tabular} & \begin{tabular}[c]{@{}c@{}}$\underline{.880}$\\[-0.2pt]{\textcolor{black!55}{$\scriptscriptstyle\pm.019$}}\end{tabular} & \begin{tabular}[c]{@{}c@{}}$-.382$\\[-0.2pt]{\textcolor{black!55}{$\scriptscriptstyle\pm.082$}}\end{tabular} & \begin{tabular}[c]{@{}c@{}}$\underline{.235}$\\[-0.2pt]{\textcolor{black!55}{$\scriptscriptstyle\pm.068$}}\end{tabular} & \begin{tabular}[c]{@{}c@{}}$\mathbf{-.721}$\\[-0.2pt]{\textcolor{black!55}{$\scriptscriptstyle\pm.133$}}\end{tabular} & \begin{tabular}[c]{@{}c@{}}$\mathbf{.554}$\\[-0.2pt]{\textcolor{black!55}{$\scriptscriptstyle\pm.139$}}\end{tabular}
    & 
    & \begin{tabular}[c]{@{}c@{}}$\mathbf{-.794}$\\[-0.2pt]{\textcolor{black!55}{$\scriptscriptstyle\pm.045$}}\end{tabular} & \begin{tabular}[c]{@{}c@{}}$\mathbf{.745}$\\[-0.2pt]{\textcolor{black!55}{$\scriptscriptstyle\pm.071$}}\end{tabular} & \begin{tabular}[c]{@{}c@{}}$-.949$\\[-0.2pt]{\textcolor{black!55}{$\scriptscriptstyle\pm.013$}}\end{tabular} & \begin{tabular}[c]{@{}c@{}}$\underline{.881}$\\[-0.2pt]{\textcolor{black!55}{$\scriptscriptstyle\pm.010$}}\end{tabular} & \begin{tabular}[c]{@{}c@{}}$\underline{-.619}$\\[-0.2pt]{\textcolor{black!55}{$\scriptscriptstyle\pm.192$}}\end{tabular} & \begin{tabular}[c]{@{}c@{}}$\underline{.473}$\\[-0.2pt]{\textcolor{black!55}{$\scriptscriptstyle\pm.182$}}\end{tabular} & \begin{tabular}[c]{@{}c@{}}$-.021$\\[-0.2pt]{\textcolor{black!55}{$\scriptscriptstyle\pm.060$}}\end{tabular} & \begin{tabular}[c]{@{}c@{}}$.004$\\[-0.2pt]{\textcolor{black!55}{$\scriptscriptstyle\pm.005$}}\end{tabular} \\
  & GroVE
    & \begin{tabular}[c]{@{}c@{}}$-.241$\\[-0.2pt]{\textcolor{black!55}{$\scriptscriptstyle\pm.151$}}\end{tabular} & \begin{tabular}[c]{@{}c@{}}$.170$\\[-0.2pt]{\textcolor{black!55}{$\scriptscriptstyle\pm.070$}}\end{tabular} & \begin{tabular}[c]{@{}c@{}}$-.149$\\[-0.2pt]{\textcolor{black!55}{$\scriptscriptstyle\pm.198$}}\end{tabular} & \begin{tabular}[c]{@{}c@{}}$.053$\\[-0.2pt]{\textcolor{black!55}{$\scriptscriptstyle\pm.044$}}\end{tabular} & \begin{tabular}[c]{@{}c@{}}$+.219$\\[-0.2pt]{\textcolor{black!55}{$\scriptscriptstyle\pm.140$}}\end{tabular} & \begin{tabular}[c]{@{}c@{}}$.090$\\[-0.2pt]{\textcolor{black!55}{$\scriptscriptstyle\pm.033$}}\end{tabular} & \begin{tabular}[c]{@{}c@{}}$-.229$\\[-0.2pt]{\textcolor{black!55}{$\scriptscriptstyle\pm.112$}}\end{tabular} & \begin{tabular}[c]{@{}c@{}}$.049$\\[-0.2pt]{\textcolor{black!55}{$\scriptscriptstyle\pm.037$}}\end{tabular}
    & 
    & \begin{tabular}[c]{@{}c@{}}$-.552$\\[-0.2pt]{\textcolor{black!55}{$\scriptscriptstyle\pm.092$}}\end{tabular} & \begin{tabular}[c]{@{}c@{}}$.414$\\[-0.2pt]{\textcolor{black!55}{$\scriptscriptstyle\pm.181$}}\end{tabular} & \begin{tabular}[c]{@{}c@{}}$-.205$\\[-0.2pt]{\textcolor{black!55}{$\scriptscriptstyle\pm.117$}}\end{tabular} & \begin{tabular}[c]{@{}c@{}}$.048$\\[-0.2pt]{\textcolor{black!55}{$\scriptscriptstyle\pm.042$}}\end{tabular} & \begin{tabular}[c]{@{}c@{}}$-.350$\\[-0.2pt]{\textcolor{black!55}{$\scriptscriptstyle\pm.130$}}\end{tabular} & \begin{tabular}[c]{@{}c@{}}$.103$\\[-0.2pt]{\textcolor{black!55}{$\scriptscriptstyle\pm.057$}}\end{tabular} & \begin{tabular}[c]{@{}c@{}}$\underline{-.428}$\\[-0.2pt]{\textcolor{black!55}{$\scriptscriptstyle\pm.131$}}\end{tabular} & \begin{tabular}[c]{@{}c@{}}$\underline{.206}$\\[-0.2pt]{\textcolor{black!55}{$\scriptscriptstyle\pm.096$}}\end{tabular} \\
  & \cellcolor{ourblue} GeoFlowVLM
    & \cellcolor{ourblue} \begin{tabular}[c]{@{}c@{}}$-.515$\\[-0.2pt]{\textcolor{black!55}{$\scriptscriptstyle\pm.117$}}\end{tabular} & \cellcolor{ourblue} \begin{tabular}[c]{@{}c@{}}$.447$\\[-0.2pt]{\textcolor{black!55}{$\scriptscriptstyle\pm.090$}}\end{tabular} & \cellcolor{ourblue} \begin{tabular}[c]{@{}c@{}}$-.905$\\[-0.2pt]{\textcolor{black!55}{$\scriptscriptstyle\pm.075$}}\end{tabular} & \cellcolor{ourblue} \begin{tabular}[c]{@{}c@{}}$.829$\\[-0.2pt]{\textcolor{black!55}{$\scriptscriptstyle\pm.082$}}\end{tabular} & \cellcolor{ourblue} \begin{tabular}[c]{@{}c@{}}$\mathbf{-.869}$\\[-0.2pt]{\textcolor{black!55}{$\scriptscriptstyle\pm.066$}}\end{tabular} & \cellcolor{ourblue} \begin{tabular}[c]{@{}c@{}}$\mathbf{.777}$\\[-0.2pt]{\textcolor{black!55}{$\scriptscriptstyle\pm.077$}}\end{tabular} & \cellcolor{ourblue} \begin{tabular}[c]{@{}c@{}}$\underline{-.690}$\\[-0.2pt]{\textcolor{black!55}{$\scriptscriptstyle\pm.083$}}\end{tabular} & \cellcolor{ourblue} \begin{tabular}[c]{@{}c@{}}$\underline{.496}$\\[-0.2pt]{\textcolor{black!55}{$\scriptscriptstyle\pm.138$}}\end{tabular}
    & \cellcolor{ourblue} 
    & \cellcolor{ourblue} \begin{tabular}[c]{@{}c@{}}$-.571$\\[-0.2pt]{\textcolor{black!55}{$\scriptscriptstyle\pm.132$}}\end{tabular} & \cellcolor{ourblue} \begin{tabular}[c]{@{}c@{}}$.474$\\[-0.2pt]{\textcolor{black!55}{$\scriptscriptstyle\pm.083$}}\end{tabular} & \cellcolor{ourblue} \begin{tabular}[c]{@{}c@{}}$\mathbf{-.961}$\\[-0.2pt]{\textcolor{black!55}{$\scriptscriptstyle\pm.023$}}\end{tabular} & \cellcolor{ourblue} \begin{tabular}[c]{@{}c@{}}$\mathbf{.887}$\\[-0.2pt]{\textcolor{black!55}{$\scriptscriptstyle\pm.011$}}\end{tabular} & \cellcolor{ourblue} \begin{tabular}[c]{@{}c@{}}$\mathbf{-.847}$\\[-0.2pt]{\textcolor{black!55}{$\scriptscriptstyle\pm.080$}}\end{tabular} & \cellcolor{ourblue} \begin{tabular}[c]{@{}c@{}}$\mathbf{.699}$\\[-0.2pt]{\textcolor{black!55}{$\scriptscriptstyle\pm.102$}}\end{tabular} & \cellcolor{ourblue} \begin{tabular}[c]{@{}c@{}}$\mathbf{-.435}$\\[-0.2pt]{\textcolor{black!55}{$\scriptscriptstyle\pm.118$}}\end{tabular} & \cellcolor{ourblue} \begin{tabular}[c]{@{}c@{}}$\mathbf{.250}$\\[-0.2pt]{\textcolor{black!55}{$\scriptscriptstyle\pm.103$}}\end{tabular} \\
\midrule
\multirow{8.68}{*}{\rotatebox[origin=c]{90}{\textbf{T2I}}}
  & ProbVLM
    & \begin{tabular}[c]{@{}c@{}}$\mathbf{-.971}$\\[-0.2pt]{\textcolor{black!55}{$\scriptscriptstyle\pm.037$}}\end{tabular} & \begin{tabular}[c]{@{}c@{}}$\mathbf{.862}$\\[-0.2pt]{\textcolor{black!55}{$\scriptscriptstyle\pm.030$}}\end{tabular} & \begin{tabular}[c]{@{}c@{}}$-.981$\\[-0.2pt]{\textcolor{black!55}{$\scriptscriptstyle\pm.022$}}\end{tabular} & \begin{tabular}[c]{@{}c@{}}$.919$\\[-0.2pt]{\textcolor{black!55}{$\scriptscriptstyle\pm.029$}}\end{tabular} & \begin{tabular}[c]{@{}c@{}}$-.498$\\[-0.2pt]{\textcolor{black!55}{$\scriptscriptstyle\pm.296$}}\end{tabular} & \begin{tabular}[c]{@{}c@{}}$.374$\\[-0.2pt]{\textcolor{black!55}{$\scriptscriptstyle\pm.246$}}\end{tabular} & \begin{tabular}[c]{@{}c@{}}$\underline{-.818}$\\[-0.2pt]{\textcolor{black!55}{$\scriptscriptstyle\pm.148$}}\end{tabular} & \begin{tabular}[c]{@{}c@{}}$\underline{.723}$\\[-0.2pt]{\textcolor{black!55}{$\scriptscriptstyle\pm.193$}}\end{tabular}
    & 
    & \begin{tabular}[c]{@{}c@{}}$-.944$\\[-0.2pt]{\textcolor{black!55}{$\scriptscriptstyle\pm.057$}}\end{tabular} & \begin{tabular}[c]{@{}c@{}}$\underline{.784}$\\[-0.2pt]{\textcolor{black!55}{$\scriptscriptstyle\pm.026$}}\end{tabular} & \begin{tabular}[c]{@{}c@{}}$-.707$\\[-0.2pt]{\textcolor{black!55}{$\scriptscriptstyle\pm.070$}}\end{tabular} & \begin{tabular}[c]{@{}c@{}}$.620$\\[-0.2pt]{\textcolor{black!55}{$\scriptscriptstyle\pm.061$}}\end{tabular} & \begin{tabular}[c]{@{}c@{}}$-.918$\\[-0.2pt]{\textcolor{black!55}{$\scriptscriptstyle\pm.060$}}\end{tabular} & \begin{tabular}[c]{@{}c@{}}$.865$\\[-0.2pt]{\textcolor{black!55}{$\scriptscriptstyle\pm.088$}}\end{tabular} & \begin{tabular}[c]{@{}c@{}}$-.947$\\[-0.2pt]{\textcolor{black!55}{$\scriptscriptstyle\pm.055$}}\end{tabular} & \begin{tabular}[c]{@{}c@{}}$.890$\\[-0.2pt]{\textcolor{black!55}{$\scriptscriptstyle\pm.069$}}\end{tabular} \\
  & AsymVLM$_{\mathrm{PS}}$
    & \begin{tabular}[c]{@{}c@{}}$-.954$\\[-0.2pt]{\textcolor{black!55}{$\scriptscriptstyle\pm.028$}}\end{tabular} & \begin{tabular}[c]{@{}c@{}}$.821$\\[-0.2pt]{\textcolor{black!55}{$\scriptscriptstyle\pm.014$}}\end{tabular} & \begin{tabular}[c]{@{}c@{}}$\underline{-.993}$\\[-0.2pt]{\textcolor{black!55}{$\scriptscriptstyle\pm.007$}}\end{tabular} & \begin{tabular}[c]{@{}c@{}}$.912$\\[-0.2pt]{\textcolor{black!55}{$\scriptscriptstyle\pm.004$}}\end{tabular} & \begin{tabular}[c]{@{}c@{}}$\underline{-.920}$\\[-0.2pt]{\textcolor{black!55}{$\scriptscriptstyle\pm.060$}}\end{tabular} & \begin{tabular}[c]{@{}c@{}}$\underline{.858}$\\[-0.2pt]{\textcolor{black!55}{$\scriptscriptstyle\pm.072$}}\end{tabular} & \begin{tabular}[c]{@{}c@{}}$+.661$\\[-0.2pt]{\textcolor{black!55}{$\scriptscriptstyle\pm.113$}}\end{tabular} & \begin{tabular}[c]{@{}c@{}}$.456$\\[-0.2pt]{\textcolor{black!55}{$\scriptscriptstyle\pm.116$}}\end{tabular}
    & 
    & \begin{tabular}[c]{@{}c@{}}$-.944$\\[-0.2pt]{\textcolor{black!55}{$\scriptscriptstyle\pm.037$}}\end{tabular} & \begin{tabular}[c]{@{}c@{}}$\mathbf{.798}$\\[-0.2pt]{\textcolor{black!55}{$\scriptscriptstyle\pm.012$}}\end{tabular} & \begin{tabular}[c]{@{}c@{}}$\underline{-.988}$\\[-0.2pt]{\textcolor{black!55}{$\scriptscriptstyle\pm.000$}}\end{tabular} & \begin{tabular}[c]{@{}c@{}}$.955$\\[-0.2pt]{\textcolor{black!55}{$\scriptscriptstyle\pm.003$}}\end{tabular} & \begin{tabular}[c]{@{}c@{}}$-.976$\\[-0.2pt]{\textcolor{black!55}{$\scriptscriptstyle\pm.015$}}\end{tabular} & \begin{tabular}[c]{@{}c@{}}$.939$\\[-0.2pt]{\textcolor{black!55}{$\scriptscriptstyle\pm.021$}}\end{tabular} & \begin{tabular}[c]{@{}c@{}}$\underline{-.976}$\\[-0.2pt]{\textcolor{black!55}{$\scriptscriptstyle\pm.012$}}\end{tabular} & \begin{tabular}[c]{@{}c@{}}$\underline{.919}$\\[-0.2pt]{\textcolor{black!55}{$\scriptscriptstyle\pm.037$}}\end{tabular} \\
  & AsymVLM$_{\mathrm{vMF}}$
    & \begin{tabular}[c]{@{}c@{}}$\underline{-.961}$\\[-0.2pt]{\textcolor{black!55}{$\scriptscriptstyle\pm.016$}}\end{tabular} & \begin{tabular}[c]{@{}c@{}}$\underline{.833}$\\[-0.2pt]{\textcolor{black!55}{$\scriptscriptstyle\pm.012$}}\end{tabular} & \begin{tabular}[c]{@{}c@{}}$\mathbf{-1.000}$\\[-0.2pt]{\textcolor{black!55}{$\scriptscriptstyle\pm.000$}}\end{tabular} & \begin{tabular}[c]{@{}c@{}}$\mathbf{.926}$\\[-0.2pt]{\textcolor{black!55}{$\scriptscriptstyle\pm.010$}}\end{tabular} & \begin{tabular}[c]{@{}c@{}}$-.762$\\[-0.2pt]{\textcolor{black!55}{$\scriptscriptstyle\pm.131$}}\end{tabular} & \begin{tabular}[c]{@{}c@{}}$.647$\\[-0.2pt]{\textcolor{black!55}{$\scriptscriptstyle\pm.163$}}\end{tabular} & \begin{tabular}[c]{@{}c@{}}$+.307$\\[-0.2pt]{\textcolor{black!55}{$\scriptscriptstyle\pm.300$}}\end{tabular} & \begin{tabular}[c]{@{}c@{}}$.182$\\[-0.2pt]{\textcolor{black!55}{$\scriptscriptstyle\pm.107$}}\end{tabular}
    & 
    & \begin{tabular}[c]{@{}c@{}}$\mathbf{-.976}$\\[-0.2pt]{\textcolor{black!55}{$\scriptscriptstyle\pm.012$}}\end{tabular} & \begin{tabular}[c]{@{}c@{}}$\mathbf{.798}$\\[-0.2pt]{\textcolor{black!55}{$\scriptscriptstyle\pm.011$}}\end{tabular} & \begin{tabular}[c]{@{}c@{}}$\mathbf{-1.000}$\\[-0.2pt]{\textcolor{black!55}{$\scriptscriptstyle\pm.000$}}\end{tabular} & \begin{tabular}[c]{@{}c@{}}$\mathbf{.970}$\\[-0.2pt]{\textcolor{black!55}{$\scriptscriptstyle\pm.004$}}\end{tabular} & \begin{tabular}[c]{@{}c@{}}$\underline{-.988}$\\[-0.2pt]{\textcolor{black!55}{$\scriptscriptstyle\pm.012$}}\end{tabular} & \begin{tabular}[c]{@{}c@{}}$\mathbf{.970}$\\[-0.2pt]{\textcolor{black!55}{$\scriptscriptstyle\pm.014$}}\end{tabular} & \begin{tabular}[c]{@{}c@{}}$-.947$\\[-0.2pt]{\textcolor{black!55}{$\scriptscriptstyle\pm.040$}}\end{tabular} & \begin{tabular}[c]{@{}c@{}}$.835$\\[-0.2pt]{\textcolor{black!55}{$\scriptscriptstyle\pm.027$}}\end{tabular} \\
  & GroVE
    & \begin{tabular}[c]{@{}c@{}}$-.345$\\[-1.6pt]{\textcolor{black!55}{$\scriptscriptstyle\pm.185$}}\end{tabular} & \begin{tabular}[c]{@{}c@{}}$.148$\\[-1.6pt]{\textcolor{black!55}{$\scriptscriptstyle\pm.137$}}\end{tabular} & \begin{tabular}[c]{@{}c@{}}$+.154$\\[-1.6pt]{\textcolor{black!55}{$\scriptscriptstyle\pm.123$}}\end{tabular} & \begin{tabular}[c]{@{}c@{}}$.059$\\[-1.6pt]{\textcolor{black!55}{$\scriptscriptstyle\pm.038$}}\end{tabular} & \begin{tabular}[c]{@{}c@{}}$-.227$\\[-1.6pt]{\textcolor{black!55}{$\scriptscriptstyle\pm.125$}}\end{tabular} & \begin{tabular}[c]{@{}c@{}}$.033$\\[-1.6pt]{\textcolor{black!55}{$\scriptscriptstyle\pm.034$}}\end{tabular} & \begin{tabular}[c]{@{}c@{}}$-.321$\\[-1.6pt]{\textcolor{black!55}{$\scriptscriptstyle\pm.086$}}\end{tabular} & \begin{tabular}[c]{@{}c@{}}$.171$\\[-1.6pt]{\textcolor{black!55}{$\scriptscriptstyle\pm.061$}}\end{tabular}
    & 
    & \begin{tabular}[c]{@{}c@{}}$-.299$\\[-1.6pt]{\textcolor{black!55}{$\scriptscriptstyle\pm.150$}}\end{tabular} & \begin{tabular}[c]{@{}c@{}}$.108$\\[-1.6pt]{\textcolor{black!55}{$\scriptscriptstyle\pm.069$}}\end{tabular} & \begin{tabular}[c]{@{}c@{}}$-.132$\\[-1.6pt]{\textcolor{black!55}{$\scriptscriptstyle\pm.249$}}\end{tabular} & \begin{tabular}[c]{@{}c@{}}$.051$\\[-1.6pt]{\textcolor{black!55}{$\scriptscriptstyle\pm.035$}}\end{tabular} & \begin{tabular}[c]{@{}c@{}}$-.052$\\[-1.6pt]{\textcolor{black!55}{$\scriptscriptstyle\pm.102$}}\end{tabular} & \begin{tabular}[c]{@{}c@{}}$.008$\\[-1.6pt]{\textcolor{black!55}{$\scriptscriptstyle\pm.009$}}\end{tabular} & \begin{tabular}[c]{@{}c@{}}$-.278$\\[-1.6pt]{\textcolor{black!55}{$\scriptscriptstyle\pm.101$}}\end{tabular} & \begin{tabular}[c]{@{}c@{}}$.074$\\[-1.6pt]{\textcolor{black!55}{$\scriptscriptstyle\pm.020$}}\end{tabular} \\
  & \cellcolor{ourblue} GeoFlowVLM
    & \cellcolor{ourblue} \begin{tabular}[c]{@{}c@{}}$-.939$\\[-0.2pt]{\textcolor{black!55}{$\scriptscriptstyle\pm.015$}}\end{tabular} & \cellcolor{ourblue} \begin{tabular}[c]{@{}c@{}}$.765$\\[-0.2pt]{\textcolor{black!55}{$\scriptscriptstyle\pm.038$}}\end{tabular} & \cellcolor{ourblue} \begin{tabular}[c]{@{}c@{}}$-.978$\\[-0.2pt]{\textcolor{black!55}{$\scriptscriptstyle\pm.049$}}\end{tabular} & \cellcolor{ourblue} \begin{tabular}[c]{@{}c@{}}$\underline{.923}$\\[-0.2pt]{\textcolor{black!55}{$\scriptscriptstyle\pm.030$}}\end{tabular} & \cellcolor{ourblue} \begin{tabular}[c]{@{}c@{}}$\mathbf{-.998}$\\[-0.2pt]{\textcolor{black!55}{$\scriptscriptstyle\pm.005$}}\end{tabular} & \cellcolor{ourblue} \begin{tabular}[c]{@{}c@{}}$\mathbf{.966}$\\[-0.2pt]{\textcolor{black!55}{$\scriptscriptstyle\pm.024$}}\end{tabular} & \cellcolor{ourblue} \begin{tabular}[c]{@{}c@{}}$\mathbf{-.981}$\\[-0.2pt]{\textcolor{black!55}{$\scriptscriptstyle\pm.016$}}\end{tabular} & \cellcolor{ourblue} \begin{tabular}[c]{@{}c@{}}$\mathbf{.937}$\\[-0.2pt]{\textcolor{black!55}{$\scriptscriptstyle\pm.003$}}\end{tabular}
    & \cellcolor{ourblue} 
    & \cellcolor{ourblue} \begin{tabular}[c]{@{}c@{}}$\underline{-.949}$\\[-0.2pt]{\textcolor{black!55}{$\scriptscriptstyle\pm.016$}}\end{tabular} & \cellcolor{ourblue} \begin{tabular}[c]{@{}c@{}}$.664$\\[-0.2pt]{\textcolor{black!55}{$\scriptscriptstyle\pm.008$}}\end{tabular} & \cellcolor{ourblue} \begin{tabular}[c]{@{}c@{}}$\mathbf{-1.000}$\\[-0.2pt]{\textcolor{black!55}{$\scriptscriptstyle\pm.000$}}\end{tabular} & \cellcolor{ourblue} \begin{tabular}[c]{@{}c@{}}$\underline{.966}$\\[-0.2pt]{\textcolor{black!55}{$\scriptscriptstyle\pm.005$}}\end{tabular} & \cellcolor{ourblue} \begin{tabular}[c]{@{}c@{}}$\mathbf{-.993}$\\[-0.2pt]{\textcolor{black!55}{$\scriptscriptstyle\pm.007$}}\end{tabular} & \cellcolor{ourblue} \begin{tabular}[c]{@{}c@{}}$\underline{.964}$\\[-0.2pt]{\textcolor{black!55}{$\scriptscriptstyle\pm.036$}}\end{tabular} & \cellcolor{ourblue} \begin{tabular}[c]{@{}c@{}}$\mathbf{-.990}$\\[-0.2pt]{\textcolor{black!55}{$\scriptscriptstyle\pm.010$}}\end{tabular} & \cellcolor{ourblue} \begin{tabular}[c]{@{}c@{}}$\mathbf{.951}$\\[-0.2pt]{\textcolor{black!55}{$\scriptscriptstyle\pm.032$}}\end{tabular} \\
\bottomrule
\end{tabular}
\end{table*}

\section{Experiments}\label{sec:experiments}
We test whether the two uncertainty scores of \cref{sec:uq} predict the errors of a frozen VLM, on datasets unseen during training.
We evaluate aleatoric and epistemic uncertainty separately, through two questions.

\noindent\textbf{(Q1)} Does the aleatoric score rank retrieval queries by how likely the VLM is to retrieve a wrong item, in both retrieval directions?

\noindent\textbf{(Q2)} Does the epistemic score improve selective zero-shot classification?

\paragraph{Shared setup.}
All methods are trained with five seeds on the same frozen embeddings of CC1M, a subset of one million image--caption pairs from Conceptual Captions~\citep{sharma2018conceptual} (Supp.~\ref{sec:supp_datasets}), and results are reported as the mean and standard deviation across seeds.
We use two backbones, CLIP ViT-H/14~\citep{clip,cherti2023reproducible} and SigLIP ViT-L/16-384~\citep{zhai2023sigmoid}, whose $\ell_2$-normalized embeddings lie on the unit hypersphere $\mathbb{S}^{d-1}$ with $d=1024$.
This shared dimensionality allows each method to use the same architecture across backbones with different pretraining objectives and data.

\paragraph{Baselines.}
We compare against methods that are trained on frozen VLM embeddings: ProbVLM~\citep{probVLM_frozen_embeddings}, AsymVLM~\citep{ju2025exploiting} and GroVE~\citep{venkataramanan2025probabilistic} for Q1, and REPVLM~\citep{ju2026epistemicuncertaintyquantificationpretrained} and ProbVLM for Q2.
Each runs with its official implementation, adapted to our embedding dimension and training set with the changes listed in Supp.~\ref{sec:supp_baselines}, which also explains this choice of baselines and the excluded methods.

\paragraph{Implementation details.}
GeoFlowVLM uses two residual MLP branches (marginal and correction) of six blocks each, trained for $600$k steps with a batch size of $8192$ on the CC1M embeddings.
We use $M=8$ for sampling (aleatoric score) and a $10$-step midpoint solver with $4$ Hutchinson probes for all log-density evaluations.
Complete implementation details are in Supp.~\ref{sec:supp_impl}.

\subsection{Aleatoric Uncertainty in Retrieval (Q1)}\label{exp:aleatoric}

\begin{figure*}[t]
    \centering
    \includegraphics[width=\textwidth]{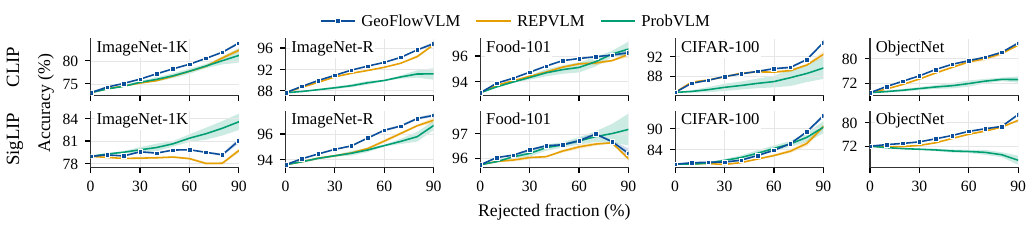}
\caption{Accuracy of the retained images against the rejected fraction, for CLIP (top) and SigLIP (bottom). Mean $\pm$ std over five seeds.}
    \label{fig:selective_clip}
\end{figure*}

\begin{table*}[t]
\centering
\caption{Acc@90\% and $\rho$ between rejected fraction and accuracy, mean $\pm$ std over five seeds. Bold: best, underline: second.}
\scriptsize
\setlength{\tabcolsep}{2.4pt}
\renewcommand{\arraystretch}{1.05}
\begin{tabular}{ll ccccc ccccc}
\toprule
& & \multicolumn{5}{c}{\textbf{Acc@90\%}$\uparrow$} & \multicolumn{5}{c}{\textbf{Spearman ($\rho$)}$\uparrow$} \\
\cmidrule(lr){3-7}\cmidrule(lr){8-12}
& \textbf{Method} & ImageNet-1K & ImageNet-R & Food-101 & CIFAR-100 & ObjectNet & ImageNet-1K & ImageNet-R & Food-101 & CIFAR-100 & ObjectNet \\
\midrule
\multirow{3}{*}{\rotatebox[origin=c]{90}{\scriptsize CLIP}} & REPVLM
  & $\underline{.823}{}_{\textcolor{black!55}{\pm.004}}$ & $\underline{.966}{}_{\textcolor{black!55}{\pm.002}}$ & $.961{}_{\textcolor{black!55}{\pm.002}}$ & $\underline{.924}{}_{\textcolor{black!55}{\pm.006}}$ & $\underline{.843}{}_{\textcolor{black!55}{\pm.001}}$ & $\mathbf{1.000}$ & $\mathbf{1.000}$ & $\underline{.995}{}_{\textcolor{black!55}{\pm.007}}$ & $\underline{.988}{}_{\textcolor{black!55}{\pm.015}}$ & $\mathbf{1.000}$ \\
 & ProbVLM
  & $.812{}_{\textcolor{black!55}{\pm.016}}$ & $.911{}_{\textcolor{black!55}{\pm.011}}$ & $\mathbf{.965}{}_{\textcolor{black!55}{\pm.006}}$ & $.896{}_{\textcolor{black!55}{\pm.022}}$ & $.732{}_{\textcolor{black!55}{\pm.010}}$ & $\underline{.998}{}_{\textcolor{black!55}{\pm.005}}$ & $\underline{.978}{}_{\textcolor{black!55}{\pm.036}}$ & $.993{}_{\textcolor{black!55}{\pm.011}}$ & $.966{}_{\textcolor{black!55}{\pm.063}}$ & $\underline{.964}{}_{\textcolor{black!55}{\pm.042}}$ \\
 & \cellcolor{ourblue} GeoFlowVLM
  & \cellcolor{ourblue} $\mathbf{.839}{}_{\textcolor{black!55}{\pm.003}}$ & \cellcolor{ourblue} $\mathbf{.968}{}_{\textcolor{black!55}{\pm.004}}$ & \cellcolor{ourblue} $\underline{.963}{}_{\textcolor{black!55}{\pm.002}}$ & \cellcolor{ourblue} $\mathbf{.948}{}_{\textcolor{black!55}{\pm.002}}$ & \cellcolor{ourblue} $\mathbf{.849}{}_{\textcolor{black!55}{\pm.005}}$ & \cellcolor{ourblue} $\mathbf{1.000}$ & \cellcolor{ourblue} $\mathbf{1.000}$ & \cellcolor{ourblue} $\mathbf{.998}{}_{\textcolor{black!55}{\pm.005}}$ & \cellcolor{ourblue} $\mathbf{.998}{}_{\textcolor{black!55}{\pm.005}}$ & \cellcolor{ourblue} $\mathbf{1.000}$ \\
\midrule
\multirow{3}{*}{\rotatebox[origin=c]{90}{\scriptsize SigLIP}} & REPVLM
  & $.797{}_{\textcolor{black!55}{\pm.003}}$ & $\underline{.971}{}_{\textcolor{black!55}{\pm.002}}$ & $.960{}_{\textcolor{black!55}{\pm.001}}$ & $\underline{.905}{}_{\textcolor{black!55}{\pm.006}}$ & $\underline{.807}{}_{\textcolor{black!55}{\pm.002}}$ & $-.219{}_{\textcolor{black!55}{\pm.117}}$ & $\mathbf{1.000}$ & $\underline{.755}{}_{\textcolor{black!55}{\pm.097}}$ & $\underline{.932}{}_{\textcolor{black!55}{\pm.018}}$ & $\underline{.968}{}_{\textcolor{black!55}{\pm.011}}$ \\
 & ProbVLM
  & $\mathbf{.835}{}_{\textcolor{black!55}{\pm.011}}$ & $.967{}_{\textcolor{black!55}{\pm.004}}$ & $\mathbf{.972}{}_{\textcolor{black!55}{\pm.006}}$ & $.903{}_{\textcolor{black!55}{\pm.018}}$ & $.673{}_{\textcolor{black!55}{\pm.015}}$ & $\mathbf{.985}{}_{\textcolor{black!55}{\pm.026}}$ & $\underline{.998}{}_{\textcolor{black!55}{\pm.005}}$ & $\mathbf{.961}{}_{\textcolor{black!55}{\pm.042}}$ & $\mathbf{.990}{}_{\textcolor{black!55}{\pm.016}}$ & $-.956{}_{\textcolor{black!55}{\pm.084}}$ \\
 & \cellcolor{ourblue} GeoFlowVLM
  & \cellcolor{ourblue} $\underline{.810}{}_{\textcolor{black!55}{\pm.004}}$ & \cellcolor{ourblue} $\mathbf{.975}{}_{\textcolor{black!55}{\pm.001}}$ & \cellcolor{ourblue} $\underline{.962}{}_{\textcolor{black!55}{\pm.002}}$ & \cellcolor{ourblue} $\mathbf{.936}{}_{\textcolor{black!55}{\pm.006}}$ & \cellcolor{ourblue} $\mathbf{.826}{}_{\textcolor{black!55}{\pm.003}}$ & \cellcolor{ourblue} $\underline{.675}{}_{\textcolor{black!55}{\pm.118}}$ & \cellcolor{ourblue} $\mathbf{1.000}$ & \cellcolor{ourblue} $.692{}_{\textcolor{black!55}{\pm.121}}$ & \cellcolor{ourblue} $\mathbf{.990}{}_{\textcolor{black!55}{\pm.016}}$ & \cellcolor{ourblue} $\mathbf{1.000}$ \\
\bottomrule
\end{tabular}
\label{tab:epistemic}
\end{table*}

\paragraph{Protocol.}
A useful aleatoric uncertainty score should assign higher uncertainty to queries with a greater risk of retrieval error.
We test this with the calibration protocol of AsymVLM~\citep{ju2025exploiting}, which extends the bin-wise evaluation of ProbVLM~\citep{probVLM_frozen_embeddings}.
The protocol sorts the queries of a retrieval dataset by uncertainty, splits them into ten bins of equal size and computes Recall@1 within each bin.
GeoFlowVLM sorts by $u_{\mathrm{aleatoric}}$ of \cref{sec:uq:aleatoric}, and each baseline by its own uncertainty.
For I2T, every method scores an image query by the uncertainty of the text it retrieves, as AsymVLM does.
For a calibrated score, Recall@1 should fall from the least to the most uncertain bin.
This trend is summarized by the Spearman correlation ($\rho$) between bin index and Recall@1, which is $-1$ when Recall@1 falls at every bin, and the coefficient of determination ($R^2$) of a linear fit, which is $1$ when it falls by the same amount each time.

\paragraph{Datasets.}
We evaluate text-to-image (T2I) and image-to-text (I2T) retrieval on four datasets: the Karpathy test split of Flickr30k~\citep{flickr_dataset,karpathy2015deep}, the val2017 split of MS-COCO~\citep{coco_dataset}, and the test classes of CUB-200~\citep{wah2011cub} and Flowers-102~\citep{nilsback2008automated} (Supp.~\ref{sec:supp_datasets}).
On Flickr30k and MS-COCO, a retrieval is correct when it returns an item paired with the query, so a text block has one correct image and an image has five correct texts. On the fine-grained CUB-200 and Flowers-102, it is correct when it returns an item of the same class~\citep{chun2021probabilistic}, so every query has many correct items.

\paragraph{Results.}
\Cref{fig:calibration} plots Recall@1 across uncertainty bins, and \cref{tab:uncertainty_calibration} summarizes the resulting curves.
Higher aleatoric uncertainty should correspond to lower Recall@1 across backbones, retrieval directions, and datasets.
GeoFlowVLM is the \emph{only method with a negative $\rho$ in all sixteen settings}; ProbVLM has a positive $\rho$ in two settings and each AsymVLM variant in one, and GroVE shows no consistent trend.
GeoFlowVLM is also the most stable across backbones: the mean absolute change in its $\rho$ between CLIP and SigLIP is $0.05$, against $0.31$ to $0.32$ for ProbVLM and each AsymVLM variant.
On Flickr30k and MS-COCO, where each text has one annotated image match, ProbVLM and AsymVLM lead in most settings, and GeoFlowVLM stays within $0.07$ of the best $\rho$ except for I2T on Flickr30k.
On CUB-200 and Flowers-102, where queries have multiple correct matches, GeoFlowVLM achieves the best $\rho$ in seven of eight settings and the second best in the remaining setting, with $\rho \leq -0.98$ for T2I on both backbones; the $R^2$ results follow a similar pattern.
On these datasets, the T2I calibration of ProbVLM and AsymVLM depends on the backbone: they weaken or reverse with CLIP, with a positive $\rho$ for both AsymVLM variants on Flowers-102, but remain close to GeoFlowVLM with SigLIP.

A possible reason for the gains on the fine-grained datasets lies in the form of the scores.
ProbVLM and AsymVLM express the ambiguity of a text through the spread of a single unimodal distribution, whereas GeoFlowVLM uses the full conditional distribution over images and compares it with the image population.
This difference should matter most when a caption fits images of several classes, as on CUB-200 and Flowers-102. 
Across all sixteen settings, GeoFlowVLM achieves the best $\rho$ in nine (one tied) and the best or second best in eleven.


\subsection{Epistemic Uncertainty in Classification (Q2)}\label{exp:epistemic}

\paragraph{Protocol.}
Following the selective-prediction protocol of REPVLM~\citep{ju2026epistemicuncertaintyquantificationpretrained}, we sort the images of each dataset by the epistemic uncertainty score of each method and compute the zero-shot accuracy after rejecting the most uncertain $0\%, 10\%, \dots, 90\%$ of the samples.
A calibrated score should assign high uncertainty to the images that the VLM misclassifies, so that the accuracy rises as more images are rejected.
We report the resulting curve, the accuracy after rejecting $90\%$ of the images (Acc@$90\%$), and the Spearman correlation ($\rho$) between the rejected fraction and the accuracy, ideally $\rho=1$.
GeoFlowVLM scores an image by $u_{\mathrm{epistemic}}$ (\cref{eq:epi_score}), and REPVLM by the sum of the negative log-densities of the image and its predicted class embedding.
ProbVLM sums the MC-dropout variances of the same two embeddings, so it also represents MC dropout, a common way to estimate epistemic uncertainty~\citep{gal2016dropout} (Supp.~\ref{sec:supp_baselines}).

\paragraph{Datasets.}
We evaluate zero-shot classification on five datasets, ImageNet-1K~\citep{deng2009imagenet}, ImageNet-R~\citep{hendrycks2021many}, Food-101~\citep{bossard2014food}, CIFAR-100~\citep{krizhevsky2009learning} and ObjectNet~\citep{barbu2019objectnet}, each on $10{,}000$ test images, which span label-set size, domain and distribution shift, with the prompt ``a photo of a \{class\}.'' for every class (Supp.~\ref{sec:supp_datasets}).

\paragraph{Results.}
\Cref{fig:selective_clip} and \cref{tab:epistemic} show that the epistemic score of GeoFlowVLM is the most stable across backbones and datasets and gives the highest Acc@$90\%$ in seven of the ten backbone--dataset settings (second to ProbVLM in the other three).
It is the \emph{only score with a positive $\rho$ in all ten settings}; the scores of REPVLM and ProbVLM each have a negative $\rho$ in one (ImageNet-1K and ObjectNet with SigLIP, respectively).
ProbVLM is also the least stable across seeds: its standard deviation of Acc@$90\%$ is the highest in every setting and about four times that of GeoFlowVLM on average.
On the shifted datasets (ImageNet-R and ObjectNet), GeoFlowVLM gives the highest Acc@$90\%$ and $\rho=1$ with both backbones.
REPVLM is the closest baseline in design, and its curves have the same shape as those of GeoFlowVLM in every setting.
Its score assumes that the image and its predicted class embedding are independent, whereas GeoFlowVLM accounts for their dependence (\cref{eq:pmi_gap}), which is consistent with the remaining gap: GeoFlowVLM exceeds REPVLM in all ten settings ($0.2$ to $3.1$ percentage points of Acc@$90\%$).
The two scores also share their weakest trends, on ImageNet-1K and Food-101 with SigLIP ($\rho=0.675$ and $0.692$ for GeoFlowVLM), where ProbVLM does not weaken.
This points to the image density term that both scores contain, and it is consistent with the limitation that REPVLM reports for in-distribution data with SigLIP~\citep{ju2026epistemicuncertaintyquantificationpretrained}.

\section{Discussion}\label{discussion_main}
This section summarizes the studies supporting our main results; protocols and full tables are in Supp.~\ref{sec:supp_results}.

\paragraph{How good are the learned densities?}
On held-out pairs, with both backbones and on every dataset, all four densities attain a far lower negative log-likelihood than the uniform density on the hypersphere, and both marginals a lower one than those of REPVLM (Supp.~\ref{sec:supp_likelihood}).
The mutual information estimated from the image side and from the text side agrees on held-out CC1M pairs, although nothing enforces it.
The conditional samples come within $0.04$ of real held-out data in precision and coverage, while the marginal samples are slightly further from it than those of REPVLM (Supp.~\ref{sec:supp_samples}).


\paragraph{Does the geometry matter?}
To isolate the role of the hypersphere, we train an otherwise identical GeoFlowVLM on $\mathcal{M}=\mathbb{R}^{d}\times\mathbb{R}^{d}$ with CLIP and compare the two on the aleatoric score, which uses both the samples and the densities of the flow.
The Euclidean version is never better (\cref{tab:ablation_geometry}): it is comparable in four settings and clearly weaker in the other four (by $0.12$ to $0.46$ in $\rho$), three of which are I2T.
All training embeddings lie on the hypersphere, yet only $2\%$ of the Euclidean samples do, so the Euclidean flow evaluates its densities outside the support of the data.

\paragraph{Does the image term detect OOD images?}
We test whether the image term of the epistemic score separates held-out CC1M images from each classification dataset, as an out-of-distribution (OOD) detector (Supp.~\ref{sec:supp_ood}).
It stays within $0.04$ of the dedicated image density of REPVLM in AUROC and FPR95 in every setting but one, and it is ahead of ProbVLM in every setting.

\paragraph{What does inference cost?}
The aleatoric score takes $6.6$\,ms per query and the epistemic score $0.69$\,ms per image (Supp.~\ref{sec:supp_scores}).
The epistemic score costs about three times that of REPVLM, and the aleatoric score far more than the parametric baselines, which need one forward pass per query and no sampling or density evaluation. This reflects the trade-off between cost and learning joint uncertainties. We also note that a setting with half the samples and function evaluations and one probe is four times faster with almost no change in ranking quality.

\begin{table}[t]
\centering
\scriptsize
\setlength{\tabcolsep}{2.5pt}
\renewcommand{\arraystretch}{0.9}
\newcommand{\sd}[1]{\\[-2pt]{\tiny\textcolor{black!55}{$\pm#1$}}}
\caption{Euclidean vs hyperspherical flow: Spearman correlation ($\rho$) of aleatoric uncertainty (CLIP, five seeds).}
\label{tab:ablation_geometry}
\begin{tabular}{@{}lcccccccc@{\hspace{4pt}}c@{}}
\toprule
& \multicolumn{2}{c}{Flickr} & \multicolumn{2}{c}{COCO} & \multicolumn{2}{c}{CUB} & \multicolumn{2}{c}{Flowers} & Sphere \\
\cmidrule(lr){2-3}\cmidrule(lr){4-5}\cmidrule(lr){6-7}\cmidrule(lr){8-9}
$\mathcal{M}$ & T2I & I2T & T2I & I2T & T2I & I2T & T2I & I2T & (\%) \\
\midrule
$(\mathbb{R}^{d})^2$ & \shortstack{$-.92$\sd{.03}} & \shortstack{$-.28$\sd{.15}} & \shortstack{$-.94$\sd{.03}} & \shortstack{$-.64$\sd{.15}} & \shortstack{$-.98$\sd{.03}} & \shortstack{$-.86$\sd{.10}} & \shortstack{$-.86$\sd{.08}} & \shortstack{$-.23$\sd{.39}} & $2$ \\
$(\mathbb{S}^{d-1})^2$ & \shortstack{$\mathbf{-.94}$\sd{.01}} & \shortstack{$\mathbf{-.52}$\sd{.12}} & \shortstack{$\mathbf{-.98}$\sd{.05}} & \shortstack{$\mathbf{-.91}$\sd{.07}} & \shortstack{$\mathbf{-1.00}$\sd{.01}} & \shortstack{$\mathbf{-.87}$\sd{.07}} & \shortstack{$\mathbf{-.98}$\sd{.02}} & \shortstack{$\mathbf{-.69}$\sd{.08}} & $\mathbf{100}$ \\
\bottomrule
\end{tabular}
\end{table}

\section{Limitations}\label{sec:limit}
The first limitation comes from the setting itself: GeoFlowVLM takes the embeddings and predictions of a frozen VLM as given, so it cannot improve them, and any error in the embeddings propagates to the densities and the scores.
Second, as for other post-hoc methods, its densities are fitted to a proxy corpus (CC1M), so the epistemic score measures how far an input lies from this corpus, and not from the unavailable pretraining data of the VLM.
Third, nothing enforces agreement between the two factorizations of the joint: the mutual information estimated from the image side and from the text side matches on held-out CC1M pairs but differs on some test sets, most on Flickr30k, where the aleatoric score for image queries also trails the baselines (\cref{exp:aleatoric}).
Finally, each score integrates an ODE, so inference is slower than for methods that score a query with a single network pass (\cref{discussion_main}).

\section{Conclusion}\label{sec:conclusion}
GeoFlowVLM learns the joint distribution of frozen image and text embeddings on the product hypersphere and derives aleatoric and epistemic scores from its densities. Our results across backbones suggest that modeling image--text dependence helps flag VLM errors beyond what individual embedding typicality reveals, with the largest retrieval gains for queries with multiple correct matches and consistent gains over marginal-based scoring in zero-shot classification. The main limitation is the computational cost of density-based uncertainty scoring. Reducing this cost while preserving uncertainty performance forms future work.

\section*{LLM Usage Statement}
We used Claude (Anthropic) and ChatGPT (OpenAI) during the preparation of this paper.
\textbf{Text.} Both tools assisted with editing, rephrasing, shortening, and reorganizing author-drafted text throughout the main paper.
\textbf{Results.} Both tools were used to help interpret the experimental results, summarize the findings, and check consistency between the text, tables, and figures. The authors verified the final numerical claims and interpretations against the original experimental results.
\textbf{Code.} Claude assisted with writing and debugging parts of the code. The authors reviewed and tested the resulting code.
The authors developed the research questions and methodology, conducted the experiments, and verified the results. They reviewed all LLM-assisted content and take full responsibility for the paper.

\iftoggle{cvprfinal}{%
\section*{Acknowledgments}
The experiments and data processing were carried out on the Arrhenius supercomputer, accessed through the National Academic Infrastructure for Supercomputing in Sweden (NAISS) under project NAISS~2026/3-305. NAISS is partially funded by the Swedish Research Council through grant agreement no.~2022-06725. Prashant Singh acknowledges funding from the Swedish Research Council under grant agreement no.~2023-05593. Prashant Singh also acknowledges support from the Knut and Alice Wallenberg Foundation through the Wallenberg AI, Autonomous Systems and Software Program (WASP) and the Data-Driven Life Science (DDLS) NEST project ``TIMED''. Ekta Vats and Tianru Zhang acknowledge support from the Kjell \& M\"{a}rta Beijer Foundation.
}{}

{
    \small
    \bibliographystyle{ieeenat_fullname}
    \bibliography{references}
}

\clearpage
\setcounter{page}{1}
\maketitlesupplementary
\appendix
\appendix
\etocdepthtag.toc{main}
\etocdepthtag.toc{supp}

\etocsettagdepth{main}{none}

\etocsettagdepth{supp}{subsection}

\tableofcontents

\section{Frequently Asked Questions}\label{sec:supp_faq}

\paragraph{F1. What does GeoFlowVLM add to REPVLM?}
GeoFlowVLM can be viewed as a generalization of REPVLM from the marginal densities of the two modalities to their joint density.
REPVLM learns one Riemannian flow on $\mathbb{S}^{d-1}$ conditioned on the modality, which gives $p(\mathbf{e}_I)$ and $p(\mathbf{e}_T)$~\citep{ju2026epistemicuncertaintyquantificationpretrained}, and GeoFlowVLM extends this flow with one model to both marginals and both conditionals, which determine the joint through \cref{eq:factorization}.
The conditionals give what the marginals cannot: how much a text narrows down the images it describes, how likely an image and a text are as a pair, samples of one modality given the other, and the mutual information of the two embeddings.
This matches what a VLM does, since it scores pairs of an image and a text, so the uncertainty of its outputs depends on the joint.
In practice, this adds an aleatoric score, which REPVLM does not provide, and an epistemic score that gives a higher Acc@$90\%$ than that of REPVLM in all ten settings of \cref{tab:epistemic}.
Learning the full joint also lets us test whether the learned densities are consistent with Bayes' rule, which the marginals alone do not allow: the two factorizations must then give the same mutual information, and they do on held-out pairs from the training distribution (\cref{tab:supp_mi_check}), although no loss term enforces it.

\paragraph{F2. Why does the mask include an absent state?}
The absent state marks a modality that is neither generated nor used as a condition, which is what a marginal needs.
The marginal of an image is the conditional averaged over all texts,
\begin{equation}
p(\mathbf{e}_I)=\int_{\mathbb{S}^{d-1}} p(\mathbf{e}_I\mid\mathbf{e}_T)\,p(\mathbf{e}_T)\,d\mathbf{e}_T,
\label{eq:supp_marginal_integral}
\end{equation}
an integral over the text sphere with no closed form.
Replacing an observed text with a placeholder such as $\mathbf{0}$ at inference does not compute it: zero lies off the sphere, and a network that has only seen real text embeddings as conditions has never learned to read it as a missing condition.
The absent state makes this case explicit: the block's input value is still zero, but the mask flags it as absent, so the network conditions on the flag and not on the value.
Training pairs whose text is marked absent at random teach the network the velocity of the marginal itself, as the null condition does in classifier-free guidance~\citep{ho2021classifierfree}.
The learned $\log p(\mathbf{e}_I)$ then follows from the change of variables of this velocity on a single sphere, up to solver and Hutchinson error.

\paragraph{F3. Why Riemannian flow matching as the generative model?}
VLM embeddings lie on spheres, where the geodesic between two points is in closed form, so Riemannian flow matching~\citep{chen2024flow} keeps training as a simulation-free regression~\citep{lipman2023flow}, whereas diffusion models on manifolds need an approximation of the heat kernel or an implicit score matching loss~\citep{de2022riemannian}.
The learned velocity gives the log-densities through the change of variables with the Riemannian divergence, which both scores need.
It is also the model family of REPVLM~\citep{ju2026epistemicuncertaintyquantificationpretrained}, so the two methods differ in the densities they learn, not in the family of generative model.
The choice is not essential: any generative model that respects the geometry and provides the four densities would do.

\section{Experimental and Implementation Details}\label{sec:supp_impl}
This section provides the details needed to reproduce our experiments.

\subsection{Datasets and Frozen Embeddings}\label{sec:supp_datasets}

\paragraph{Training data.}
GeoFlowVLM and all baselines are trained on the same dataset, CC1M, a subset of one million image--caption pairs from the Conceptual Captions 3M WebDataset release (\texttt{pixparse/cc3m-wds})~\citep{sharma2018conceptual}, in which each image has one web-crawled caption.
To remove low-quality pairs, we discard images whose height or width is below $224$ pixels (smaller than the standard CLIP input), images whose width-to-height ratio is above $3$ or below $1/3$ (heavily cropped or distorted when resized to a square input), and captions with fewer than $3$ or more than $100$ words (near-empty or malformed).
CC1M consists of the first one million pairs in the release that pass these filters. The training and validation splits of CC1M are described in Supp.~\ref{sec:supp_training}.

\paragraph{VLM backbones.}
To test the generality of GeoFlowVLM across embedding spaces, we use two frozen VLM backbones that differ in training loss and pretraining data.
OpenCLIP ViT-H/14 (\texttt{laion2b\_s32b\_b79k} checkpoint)~\citep{cherti2023reproducible} is trained with the softmax contrastive loss of CLIP~\citep{clip} on LAION-2B~\citep{schuhmann2022laion}.
SigLIP ViT-L/16 at $384$-pixel resolution (\texttt{webli} checkpoint)~\citep{zhai2023sigmoid} is trained with a pairwise sigmoid loss on WebLI.
With each backbone, we encode all images and captions of CC1M once, using the backbone's own preprocessing and tokenizer.
Both losses compare embeddings by cosine similarity, so we $\ell_2$-normalize the embeddings, and every image and text embedding lies on the unit hypersphere $\mathbb{S}^{d-1}$, with $d=1024$ for both backbones.
GeoFlowVLM and all baselines are trained on these precomputed embeddings, and the evaluation datasets are encoded in the same way.

\paragraph{Retrieval datasets (Q1).}
The four retrieval datasets of \cref{exp:aleatoric}, Flickr30k~\citep{flickr_dataset,karpathy2015deep}, MS-COCO~\citep{coco_dataset}, CUB-200~\citep{wah2011cub} and Flowers-102~\citep{nilsback2008automated}, cover two regimes of ambiguity: on Flickr30k and MS-COCO each image has five independently written captions, so a text has one correct image and an image has five correct texts, whereas on CUB-200 and Flowers-102 correctness is defined by the class, so every query has many correct items.
CUB-200 uses the $50$ test species of PCME~\citep{chun2021probabilistic} and Flowers-102 the $20$ test classes of GroVE~\citep{venkataramanan2025probabilistic}, each image with the ten captions of \citet{reed2016learning}; the class-level correctness rule follows the CUB benchmark of PCME, which ProbVLM~\citep{probVLM_frozen_embeddings} and GroVE also adopt.
CUB-200 and Flowers-102 add fine-grained categories and a domain shift from CC1M.
Table~\ref{tab:datasets} lists splits and sizes.

\begin{table}[t]
\centering
\small
\setlength{\tabcolsep}{3pt}
\caption{Datasets. Class counts refer to the subsets we use.}
\label{tab:datasets}
\begin{tabular}{@{}llrrr@{}}
\toprule
Dataset & Split & Classes & Caps./img & Images \\
\midrule
\multicolumn{5}{@{}l}{\emph{Training}} \\
CC1M        & train         & --      & 1  & 990{,}000 \\
CC1M        & validation    & --      & 1  & 10{,}000 \\
\addlinespace[2pt]
\multicolumn{5}{@{}l}{\emph{Retrieval (Q1)}} \\
MS-COCO     & val2017       & --      & 5  & 5{,}000 \\
Flickr30k   & Karpathy test & --      & 5  & 1{,}000 \\
CUB-200     & test          & 50      & 10 & 2{,}967 \\
Flowers-102 & test          & 20      & 10 & 1{,}155 \\
\addlinespace[2pt]
\multicolumn{5}{@{}l}{\emph{Classification (Q2)}} \\
ImageNet-1K & val           & 1{,}000 & -- & 10{,}000 \\
ImageNet-R  & test          & 200     & -- & 10{,}000 \\
Food-101    & test          & 101     & -- & 10{,}000 \\
CIFAR-100   & test          & 100     & -- & 10{,}000 \\
ObjectNet   & test          & 113     & -- & 10{,}000 \\
\bottomrule
\end{tabular}
\end{table}

\paragraph{Classification datasets (Q2).}
The five classification datasets of \cref{exp:epistemic}, ImageNet-1K~\citep{deng2009imagenet}, ImageNet-R~\citep{hendrycks2021many}, Food-101~\citep{bossard2014food}, CIFAR-100~\citep{krizhevsky2009learning} and ObjectNet~\citep{barbu2019objectnet}, each depart from the CC1M training data in a different way, which epistemic uncertainty should detect.
ImageNet-1K contains object-centric photographs, ImageNet-R renditions such as paintings, cartoons and sketches of $200$ ImageNet-1K classes, Food-101 fine-grained food categories, CIFAR-100 images of $32 \times 32$ pixels, below the $224$-pixel minimum of CC1M, and ObjectNet objects in unusual poses, backgrounds and viewpoints, restricted to the $113$ classes it shares with ImageNet-1K.
We use all $10{,}000$ CIFAR-100 test images and draw $10{,}000$ images uniformly at random, with a fixed seed, from the ImageNet-1K validation split, the $30{,}000$ images of ImageNet-R and the Food-101 and ObjectNet test splits.
For OOD detection, the in-distribution set consists of $10{,}000$ held-out CC1M images, and each of the five datasets serves as a shifted set (Supp.~\ref{sec:supp_ood}).
Table~\ref{tab:datasets} summarizes all datasets.

\subsection{GeoFlowVLM Implementation}\label{sec:supp_architecture}

\paragraph{Consistency check.}
We first establish a check that any joint model such as GeoFlowVLM should satisfy: its four learned densities must agree with one joint distribution.
Conditioning on $\mathbf{e}_T$ cannot increase the differential entropy of $\mathbf{e}_I$, that is, $H(\mathbf{e}_I\mid\mathbf{e}_T)\le H(\mathbf{e}_I)$ with $H(\mathbf{e}_I)=-\mathbb{E}[\log p(\mathbf{e}_I)]$, and the gap between the two entropies is the expected gain in log-density from conditioning.
For the true densities this gain is the same in both directions and equals the non-negative mutual information~\citep{cover2006elements},
\begin{equation}
\begin{aligned}
I(\mathbf{e}_I;\mathbf{e}_T)
&= \mathbb{E}\big[\log p(\mathbf{e}_I\mid\mathbf{e}_T)-\log p(\mathbf{e}_I)\big]\\
&= \mathbb{E}\big[\log p(\mathbf{e}_T\mid\mathbf{e}_I)-\log p(\mathbf{e}_T)\big] \ge 0.
\end{aligned}
\label{eq:supp_mi_consistency}
\end{equation}
We call the two expectations the \emph{image gain} and the \emph{text gain}.
Using the log-density estimates of Supp.~\ref{sec:supp_sampling_density}, we monitor both gains on held-out image--text pairs throughout training.
Our initial model, a single branch shared across all four densities, gave negative and unequal gains (Table~\ref{tab:supp_mi_check}).
We attribute this to the marginal and the conditional of the same modality competing for the same weights, which motivates the separation described next.

\paragraph{Marginal and correction branches.}
To remove this competition, we split the velocity model into a marginal branch and a correction branch with the same architecture (\cref{sec:method}).
The marginal branch $\mathbf{v}_t^{\theta}$ learns the velocities of the two marginals, $p(\mathbf{e}_I)$ and $p(\mathbf{e}_T)$.
The correction branch $\boldsymbol{\Delta}_t^{\psi}$ learns the adjustment that turns a marginal velocity into the velocity of the corresponding conditional, $p(\mathbf{e}_I\mid\mathbf{e}_T)$ or $p(\mathbf{e}_T\mid\mathbf{e}_I)$.
For a conditional, the velocity is
\begin{equation}
\mathbf{v}_t^{\phi}(\mathbf{z}_t,\mathbf{m}) = \mathbf{v}_t^{\theta}(\bar{\mathbf{z}}_t,\bar{\mathbf{m}}) + \boldsymbol{\Delta}_t^{\psi}(\mathbf{z}_t,\mathbf{m}),
\label{eq:supp_conditional_correction}
\end{equation}
where $\phi=(\theta,\psi)$, and for a marginal it is $\mathbf{v}_t^{\theta}$ alone.
The sum is projected onto the tangent space with $\mathrm{Proj}_{\mathbf{z}_t}$, as in \cref{eq:residual}; since the projection is linear, this equals projecting each branch separately.
The inputs $(\bar{\mathbf{z}}_t,\bar{\mathbf{m}})$ replace the observed block with the absent representation, so the marginal branch never receives the observed embedding.
Both branches are trained jointly with the objective of \cref{eq:loss}, in which we block the gradients of the conditional terms with respect to $\theta$, so the marginal branch is updated only by the marginal terms. The correction branch has a zero-initialized output layer, so each conditional velocity starts equal to its marginal velocity.
The additive form is inspired by classifier guidance, which uses Bayes' rule to write a conditional score as a marginal score plus a likelihood score~\citep{dhariwal2021diffusion}, and we use the analogous decomposition for velocities.
ControlNet similarly adds conditioning through a zero-initialized branch on a base network that the conditioning does not update~\citep{zhang2023adding}.
With this split, the gains went from $-24.1$ vs $-487.8$ nats to $34.5$ vs $35.6$ nats, positive and nearly equal (Table~\ref{tab:supp_mi_check}).

\paragraph{Capacity control.}
A fair comparison must rule out capacity as the reason the split passes the check, because the split doubles the parameters of the single branch from $42$M to $84$M.
We test this with two additional versions (Table~\ref{tab:supp_mi_check}).
The first is a single branch with $14$ instead of $6$ blocks, which matches the $84$M parameters of the split model.
The second is a split model with $2$ blocks per branch, which matches the $42$M parameters of the original single branch.
The deeper single branch still failed the check, with gains of $-4.6$ vs $15.3$ nats, while the shallower split model passed it, with $60.8$ vs $63.0$ nats.
The split, not the added capacity, therefore removes the violation.
The gains depend on how well each density is fit, so we compare only their sign and agreement across versions.
At $100$k steps, the GeoFlowVLM configuration gives $34.5$ vs $35.6$ nats, positive and nearly equal.
After the full $600$k steps, it satisfies the check as closely with both backbones, with $27$ vs $28$ nats for CLIP and $30$ vs $30$ nats for SigLIP (\cref{tab:supp_likelihood}).

\paragraph{Architecture.}
Each branch of the velocity model is a residual MLP with six blocks of width $1024$, SiLU activations and $42$M parameters.
It takes the joint state $\mathbf{z}_t\in\mathbb{R}^{2d}$, the mask $\mathbf{m}$ and the time $t$, and outputs the velocity of the generated block.
Its first layer receives $\mathbf{z}_t$ together with one-hot encodings of the role of each block, observed, generated or absent.
All blocks share one conditioning vector, which sums three $1024$-dimensional terms, an embedding of $t$, a learned embedding of $\mathbf{m}$ and a linear projection of the observed block.
AdaLN-Zero~\citep{peebles2023scalable} maps this vector to a block-specific scale, shift and gate.
Each block scales and shifts its normalized hidden state, passes the result through a two-layer MLP and adds the gated output through a residual connection.
The output layer applies its own scale and shift, followed by a linear map to $\mathbb{R}^{2d}$.
We keep the generated block as the branch output.
The residual gates and the output layer are initialized to zero, so each branch starts as the zero velocity field.

\paragraph{Training.}\label{sec:supp_training}
Each training step draws a batch of CC1M pairs and assigns every pair one of the four masks uniformly at random, which is $p(\mathbf{m})$ in \cref{eq:loss}, and its own time $t\sim\mathcal{U}[0,1]$, so every batch trains all four densities.
We hold out a random $1\%$ of CC1M ($10{,}000$ pairs), the same for all seeds, for validation; the consistency check uses a fixed $2{,}000$-pair subset of it.
We train for $600$k steps and evaluate the exponential moving average (EMA) of the weights at the final step, whose held-out likelihood on this subset was the best among the checkpoints saved every $50$k steps in every run.
Table~\ref{tab:supp_geoflow_training} lists the optimization settings, which are the same for both the backbones.

\begin{table}[t]
\centering
\small
\setlength{\tabcolsep}{3pt}
\caption{Image and text gains (Eq.~\eqref{eq:supp_mi_consistency}, nats) on $2{,}000$ held-out CC1M pairs after $100$k steps with CLIP. The check passes when both gains are positive and nearly equal. Depth: residual blocks of width $1024$ per branch. Shaded: GeoFlowVLM final configuration.}
\label{tab:supp_mi_check}
\begin{tabular}{@{}lcrrrc@{}}
\toprule
Version & Depth & Params & Image & Text & Check \\
\midrule
Single branch & 6 & 42M & $-24.1$ & $-487.8$ & \ding{55} \\
Single, deeper & 14 & 84M & $-4.6$ & $15.3$ & \ding{55} \\
Split, shallower & 2+2 & 42M & $60.8$ & $63.0$ & \checkmark \\
\rowcolor{gray!15}
Split & 6+6 & 84M & $34.5$ & $35.6$ & \checkmark \\
\bottomrule
\end{tabular}
\end{table}

\begin{table}[t]
\centering
\small
\setlength{\tabcolsep}{4pt}
\caption{GeoFlowVLM training settings, the same for the CLIP and SigLIP backbones.}
\label{tab:supp_geoflow_training}
\begin{tabular}{@{}lr@{\hspace{1.5em}}lr@{}}
\toprule
Setting & Value & Setting & Value \\
\midrule
Optimizer & AdamW & Batch size & $8{,}192$ \\
Learning rate & $3\times10^{-4}$ & Steps & $600{,}000$ \\
Weight decay & $0.01$ & Grad.\ clip norm & $1.0$ \\
Warm-up (linear) & $2{,}000$ steps & Precision & bfloat16 mixed \\
LR decay & cosine to $0$ & EMA decay & $0.9999$ \\
\bottomrule
\end{tabular}
\end{table}

\subsection{Sampling and Log-Density Evaluation}\label{sec:supp_sampling_density}

\paragraph{Inference ODE.}
Inference with Riemannian flow matching~\citep{chen2024flow} requires solving an ODE on the manifold, here $\mathbb{S}^{d-1}$, along which the generated block moves with velocity $\mathrm{Proj}_{\mathbf{z}_t}\big(\mathbf{v}_t^{\phi}(\mathbf{z}_t,\mathbf{m})\big)$ while the observed block stays fixed; the projection of \cref{eq:residual} is applied at every evaluation of the velocity.
Sampling integrates this ODE \emph{forward}, from a uniform draw on $\mathbb{S}^{d-1}$ at $t=0$ to an embedding at $t=1$.
Log-density evaluation integrates it \emph{backward}, from a given embedding at $t=1$ to $t=0$.

\paragraph{Solver.}
We use the midpoint method, which is second-order accurate at two velocity evaluations per step, and take each step with the exponential map along a geodesic of $\mathbb{S}^{d-1}$, so the state stays exactly on the sphere.
The $n=10$ steps follow a quadratic grid~\citep{song2021denoising}, $t_k=1-(1-k/n)^2$ for $k=0,\ldots,n$, which places shorter steps near the data at $t=1$, where the Jacobian of the target velocity grows as $1/(1-t)$ and the field changes fastest.
One integration therefore costs $20$ function evaluations (NFE).

\paragraph{Log-densities.}
Following \citet[App.~H]{chen2024flow}, the log-density follows from the change-of-variables formula on the sphere, which subtracts the time integral of the Riemannian divergence of the velocity from the log-density of the uniform distribution.
For the conditional of an image given a text,
\begin{equation}
\log p(\mathbf{e}_I\mid\mathbf{e}_T)=-\log\big|\mathbb{S}^{d-1}\big|-\int_0^1\operatorname{div}\,\mathbf{v}_t^{\phi}(\mathbf{z}_t,\mathbf{m})\,\mathrm{d}t,
\label{eq:supp_logdensity}
\end{equation}
where $\mathbf{z}_t$ follows the ODE backward from $\mathbf{z}_1=[\mathbf{e}_I;\mathbf{e}_T]$ and the divergence is taken with respect to the image block, which the flow generates.
The exact divergence is the trace of the Jacobian $\mathbf{J}$ of the generated block's velocity on its tangent space, which needs $d=1024$ vector--Jacobian products at every function evaluation.
We therefore use Hutchinson's estimator~\citep{hutchinson1989stochastic},
\begin{equation}
\operatorname{tr}(\mathbf{J}) \approx \frac{1}{L}\sum_{l=1}^{L}\boldsymbol{\xi}_l^{\top}\mathbf{J}\,\boldsymbol{\xi}_l,
\end{equation}
which is unbiased and needs one vector--Jacobian product per probe $\boldsymbol{\xi}_l$.
We draw $L=4$ Rademacher probes and project them onto the tangent space so that the estimate targets the Riemannian divergence.
This reduces the vector--Jacobian products per function evaluation from $1024$ to $4$.

\begin{table*}[t]
\centering
\scriptsize
\setlength{\tabcolsep}{3.5pt}
\renewcommand{\arraystretch}{0.95}
\newcommand{\sd}[1]{\\[-2pt]{\tiny\textcolor{black!55}{$\pm#1$}}}
\caption{Held-out NLL per dimension (nats, $\downarrow$), mean (std below) over five seeds, $5{,}000$ pairs per dataset. CC1M: pairs held out from training; others: Q1 test pairs. MI: mutual information in nats, image side / text side, mean. Bold: lower marginal NLL.}
\label{tab:supp_likelihood}
\begin{tabular}{@{}llccccc@{\hspace{18pt}}ccccc@{}}
\toprule
& & \multicolumn{5}{c}{\textbf{CLIP ViT-H/14}} & \multicolumn{5}{c}{\textbf{SigLIP ViT-L/16-384}} \\
\cmidrule(lr){3-7}\cmidrule(l){8-12}
\textbf{Density} & \textbf{Model} & CC1M & Flickr30k & MS-COCO & CUB-200 & Flowers-102 & CC1M & Flickr30k & MS-COCO & CUB-200 & Flowers-102 \\
\midrule
Uniform & -- & $-2.05$ & $-2.05$ & $-2.05$ & $-2.05$ & $-2.05$ & $-2.05$ & $-2.05$ & $-2.05$ & $-2.05$ & $-2.05$ \\
\addlinespace[3pt]
$p(\mathbf{e}_I)$ & REPVLM & \shortstack{$-3.49$\sd{.01}} & \shortstack{$-3.35$\sd{.01}} & \shortstack{$-3.38$\sd{.01}} & \shortstack{$-3.46$\sd{.01}} & \shortstack{$-3.51$\sd{.01}} & \shortstack{$-3.59$\sd{.01}} & \shortstack{$-3.57$\sd{.01}} & \shortstack{$-3.55$\sd{.01}} & \shortstack{$-3.50$\sd{.01}} & \shortstack{$-3.52$\sd{.01}} \\

 & GeoFlowVLM & \shortstack{$\mathbf{-4.01}$\sd{.00}} & \shortstack{$\mathbf{-3.90}$\sd{.00}} & \shortstack{$\mathbf{-3.91}$\sd{.00}} & \shortstack{$\mathbf{-4.02}$\sd{.00}} & \shortstack{$\mathbf{-4.04}$\sd{.00}} & \shortstack{$\mathbf{-3.83}$\sd{.00}} & \shortstack{$\mathbf{-3.82}$\sd{.00}} & \shortstack{$\mathbf{-3.80}$\sd{.00}} & \shortstack{$\mathbf{-3.78}$\sd{.00}} & \shortstack{$\mathbf{-3.79}$\sd{.00}} \\
\addlinespace[3pt]
$p(\mathbf{e}_T)$ & REPVLM & \shortstack{$-3.63$\sd{.00}} & \shortstack{$-3.60$\sd{.01}} & \shortstack{$-3.59$\sd{.00}} & \shortstack{$-3.34$\sd{.01}} & \shortstack{$-3.41$\sd{.01}} & \shortstack{$-3.71$\sd{.01}} & \shortstack{$-3.53$\sd{.00}} & \shortstack{$-3.54$\sd{.00}} & \shortstack{$-3.41$\sd{.01}} & \shortstack{$-3.44$\sd{.01}} \\

 & GeoFlowVLM & \shortstack{$\mathbf{-4.09}$\sd{.01}} & \shortstack{$\mathbf{-4.10}$\sd{.01}} & \shortstack{$\mathbf{-4.09}$\sd{.01}} & \shortstack{$\mathbf{-3.86}$\sd{.00}} & \shortstack{$\mathbf{-3.89}$\sd{.00}} & \shortstack{$\mathbf{-4.10}$\sd{.00}} & \shortstack{$\mathbf{-3.96}$\sd{.00}} & \shortstack{$\mathbf{-3.95}$\sd{.00}} & \shortstack{$\mathbf{-3.83}$\sd{.01}} & \shortstack{$\mathbf{-3.85}$\sd{.00}} \\
\addlinespace[3pt]

$p(\mathbf{e}_I\mid\mathbf{e}_T)$ & GeoFlowVLM & \shortstack{$-4.03$\sd{.00}} & \shortstack{$-3.91$\sd{.00}} & \shortstack{$-3.93$\sd{.00}} & \shortstack{$-4.03$\sd{.00}} & \shortstack{$-4.06$\sd{.00}} & \shortstack{$-3.86$\sd{.00}} & \shortstack{$-3.83$\sd{.00}} & \shortstack{$-3.82$\sd{.00}} & \shortstack{$-3.80$\sd{.00}} & \shortstack{$-3.81$\sd{.00}} \\

$p(\mathbf{e}_T\mid\mathbf{e}_I)$ & GeoFlowVLM & \shortstack{$-4.12$\sd{.00}} & \shortstack{$-4.11$\sd{.01}} & \shortstack{$-4.11$\sd{.01}} & \shortstack{$-3.87$\sd{.01}} & \shortstack{$-3.91$\sd{.00}} & \shortstack{$-4.13$\sd{.00}} & \shortstack{$-3.96$\sd{.00}} & \shortstack{$-3.97$\sd{.00}} & \shortstack{$-3.85$\sd{.01}} & \shortstack{$-3.86$\sd{.00}} \\
\addlinespace[3pt]

MI & GeoFlowVLM & $27/28$ & $15/7$ & $20/15$ & $17/18$ & $18/18$ & $30/30$ & $18/7$ & $23/17$ & $23/13$ & $18/16$ \\
\bottomrule
\end{tabular}
\end{table*}

\subsection{Baseline Implementations}\label{sec:supp_baselines}

\paragraph{Baseline selection.}
GeoFlowVLM trains a probabilistic model on the output embeddings of frozen encoders, so we compare it with the methods that do the same, the second group of \cref{sec:related_works}, which gives the four baselines of the main text.
The first group trains the encoders or needs their internal layers, which this setting does not allow, and the training-free methods train no probabilistic model at all.
Our question is what a probabilistic model on frozen embeddings can provide, and how the choice of this model matters.
The four baselines span this choice: they differ in the \emph{distribution family}, in the \emph{space} in which they model the embeddings, and in \emph{what they model}.
All methods share the embeddings and the training data, so the comparison is about the trained probabilistic model of each method.

\paragraph{Shared protocol.}
We follow the official implementation of each baseline and adapt it to our backbones, with embedding dimension $d=1024$, and to CC1M as the training set.
Each baseline uses the same settings for both backbones, and the same CC1M split and seeds as GeoFlowVLM (Supp.~\ref{sec:supp_training}).
The paragraphs below describe in detail how each baseline is trained.

\paragraph{ProbVLM~\citep{probVLM_frozen_embeddings}.}
ProbVLM models aleatoric uncertainty with a generalized Gaussian per embedding, predicted by a residual MLP adapter per modality, and epistemic uncertainty with Monte Carlo dropout.
We keep its learning rate of $10^{-4}$ and dropout rate of $0.1$, and follow the released REPVLM code~\citep{ju2026epistemicuncertaintyquantificationpretrained} for the remaining settings: a hidden width of $2d$, and Adam with batch size $256$ for $200$k steps, with a linear warm-up over $20$k steps followed by a constant rate.
For Q1, we switch dropout off, use the predicted variance alone and rank text queries by its inverse, which orders the bins as for the other methods.
For Q2, we use the variance of the predicted mean over $15$ dropout passes alone, at the training dropout rate of $0.1$.

\paragraph{AsymVLM~\citep{ju2025exploiting}.}
AsymVLM builds on the observation that aleatoric uncertainty is asymmetric across the modalities: a text can describe many images, whereas an image shows one specific scene.
It therefore maps each text embedding to a distribution on $\mathbb{S}^{d-1}$ and keeps image embeddings deterministic, and we evaluate both of its versions, with a von Mises--Fisher (AsymVLM$_{\mathrm{vMF}}$) and a power-spherical (AsymVLM$_{\mathrm{PS}}$) distribution.
Accordingly, its evaluation scores a text-to-image query by the uncertainty of its text and an image-to-text query by that of its top-ranked retrieved text, and we apply this rule to all methods.
We add a $1{,}000$-step linear warm-up to its schedule and train both versions at its batch size of $2048$ for $120$k steps.

\paragraph{GroVE~\citep{venkataramanan2025probabilistic}.}
GroVE models the aleatoric uncertainty of both modalities with a Gaussian process latent variable model.
It learns one latent point per training pair, maps it to the image and text embeddings with sparse Gaussian processes, and scores an embedding by its mean predictive variance after inferring the embedding's latent point at test time.
We train it with the training loop and hyperparameters of its official code, using Adam with a constant learning rate of $10^{-5}$ and batch size $128$, for at most $1.1$M steps.
We add early stopping on $2{,}000$ held-out pairs of the CC1M validation split (Supp.~\ref{sec:supp_training}): every $2{,}000$ steps, we evaluate GroVE's objective on them, inferring their latent points as at test time, stop after $10$ evaluations without a $0.1\%$ relative improvement, and keep the checkpoint with the lowest validation loss.


\paragraph{REPVLM~\citep{ju2026epistemicuncertaintyquantificationpretrained}.}
REPVLM models epistemic uncertainty with the density of one Riemannian flow on $\mathbb{S}^{d-1}$ conditioned on the modality label $c$.
It scores an embedding by $-\log p(\mathbf{e}\mid c)$, computed with $10$ Euler steps and $5$ Hutchinson probes.
For selective prediction, its implementation scores an image by the sum of its own uncertainty and that of its predicted class embedding, and we apply the same rule to ProbVLM; for OOD detection, it scores an image by its own uncertainty alone, and we follow this.
We use the architecture and training settings of its released code, with batches of $512$ image--text pairs for $400$k steps.

\subsection{Compute and Software Environment}\label{sec:supp_compute}
All experiments ran on an HPC cluster with NVIDIA GH200 Grace Hopper Superchips, each pairing a $72$-core Grace CPU with a Hopper GPU with $96$\,GB of memory, and each training run uses one GPU.
All methods run in one software environment with Python 3.12, PyTorch 2.14 with CUDA 13.0 and cuDNN 9.24.

\section{Additional Results}\label{sec:supp_results}

\subsection{Quality of the Learned Densities}\label{sec:supp_likelihood}
The uncertainty scores of GeoFlowVLM are functionals of its learned distributions, so they are meaningful only if these distributions fit the data.

\paragraph{Protocol.} A natural test of this fit is the NLL of held-out image--text embedding pairs, computed as in Supp.~\ref{sec:supp_sampling_density}, which is lower for a density that fits the data better.
As a reference, we use the uniform distribution on $\mathbb{S}^{d-1}$, whose NLL is $\log|\mathbb{S}^{d-1}|$ for every embedding: a density that has learned nothing about the data stays at this level, so the gap below it shows how much structure a density captures.
Among the baselines, only REPVLM learns a density of the embeddings on $\mathbb{S}^{d-1}$, for the two marginals, which we compute with its official solver; the others define densities in $\mathbb{R}^d$ or fit them with a contrastive loss, so their likelihoods are not comparable.
We also estimate the mutual information from the image side, $\mathbb{E}[\log p(\mathbf{e}_I\mid\mathbf{e}_T)-\log p(\mathbf{e}_I)]$, and from the text side, $\mathbb{E}[\log p(\mathbf{e}_T\mid\mathbf{e}_I)-\log p(\mathbf{e}_T)]$; for densities that come from one joint distribution, the two estimates should be positive and equal (Supp.~\ref{sec:supp_architecture}).
We evaluate on $5{,}000$ held-out CC1M pairs and $5{,}000$ test pairs from each of the four retrieval datasets of Q1, with both backbones and all five seeds, and report the mean and standard deviation over seeds.

\paragraph{Results.} The NLL of all four densities of GeoFlowVLM lies far below the uniform level on every dataset and both backbones (\cref{tab:supp_likelihood}), with negligible variation across seeds.
GeoFlowVLM also attains a lower NLL than REPVLM for both marginals on every dataset and both backbones, although REPVLM learns only these two.
The two mutual-information estimates agree best on held-out CC1M pairs, from the training distribution, where they differ by at most $1$ nat, as \cref{sec:method} expects at the optimum.
On the test sets, the agreement is best on Flowers-102 (at most $2$ nats) and on CUB-200 with CLIP ($1$ nat), weaker on MS-COCO ($5$ to $6$ nats), and weakest on Flickr30k ($8$ and $11$ nats) and on CUB-200 with SigLIP ($10$ nats).
Where the two differ by more than $2$ nats, the text-side estimate is the lower one, and both remain positive.
Exact agreement is not expected off the training distribution, where the two conditionals are evaluated on pairs unlike those they were fitted to.

\begin{table*}[t]
\centering
\footnotesize
\setlength{\tabcolsep}{3.5pt}
\newcommand{\sd}[1]{{}_{\scriptscriptstyle\textcolor{black!55}{\pm#1}}}
\caption{OOD detection of each OOD set against held-out CC1M images, mean $\pm$ std over five seeds. All methods score an image by its image term alone. AUROC $\uparrow$, FPR95 $\downarrow$. Bold: best, underline: second best, per backbone.}
\label{tab:supp_ood}
\begin{tabular}{@{}clcccccccccc@{}}
\toprule
& & \multicolumn{2}{c}{ImageNet-1K} & \multicolumn{2}{c}{ImageNet-R} & \multicolumn{2}{c}{Food-101} & \multicolumn{2}{c}{CIFAR-100} & \multicolumn{2}{c}{ObjectNet} \\
\cmidrule(lr){3-4}\cmidrule(lr){5-6}\cmidrule(lr){7-8}\cmidrule(lr){9-10}\cmidrule(l){11-12}
& \textbf{Method} & AUROC & FPR95 & AUROC & FPR95 & AUROC & FPR95 & AUROC & FPR95 & AUROC & FPR95 \\
\midrule
& REPVLM & $\mathbf{.722}\sd{.002}$ & $\mathbf{.884}\sd{.003}$ & $\mathbf{.835}\sd{.001}$ & $\mathbf{.750}\sd{.004}$ & $\underline{.793}\sd{.004}$ & $\mathbf{.983}\sd{.001}$ & $\mathbf{.994}\sd{.001}$ & $\mathbf{.007}\sd{.002}$ & $\mathbf{.873}\sd{.003}$ & $\mathbf{.898}\sd{.006}$ \\
& ProbVLM & $.560\sd{.011}$ & $.966\sd{.004}$ & $.694\sd{.010}$ & $.878\sd{.007}$ & $.656\sd{.028}$ & $.988\sd{.005}$ & $.957\sd{.006}$ & $.292\sd{.059}$ & $.663\sd{.012}$ & $.987\sd{.006}$ \\
\rowcolor{gray!15}
\multirow{-3}{*}{\rotatebox[origin=c]{90}{CLIP}} & GeoFlowVLM & $\underline{.704}\sd{.001}$ & $\underline{.916}\sd{.001}$ & $\underline{.810}\sd{.001}$ & $\underline{.823}\sd{.004}$ & $\mathbf{.812}\sd{.001}$ & $\underline{.987}\sd{.000}$ & $\underline{.984}\sd{.001}$ & $\underline{.047}\sd{.011}$ & $\underline{.864}\sd{.001}$ & $\underline{.935}\sd{.002}$ \\
\midrule
& REPVLM & $\mathbf{.741}\sd{.002}$ & $\mathbf{.903}\sd{.001}$ & $\underline{.790}\sd{.000}$ & $\underline{.823}\sd{.002}$ & $\underline{.753}\sd{.002}$ & $\mathbf{.979}\sd{.001}$ & $\mathbf{.798}\sd{.005}$ & $\underline{.931}\sd{.003}$ & $\underline{.832}\sd{.002}$ & $\underline{.917}\sd{.002}$ \\
& ProbVLM & $.597\sd{.017}$ & $.930\sd{.009}$ & $.720\sd{.014}$ & $.876\sd{.014}$ & $.596\sd{.024}$ & $.987\sd{.002}$ & $.589\sd{.026}$ & $.933\sd{.019}$ & $.738\sd{.010}$ & $.929\sd{.014}$ \\
\rowcolor{gray!15}
\multirow{-3}{*}{\rotatebox[origin=c]{90}{SigLIP}} & GeoFlowVLM & $\underline{.727}\sd{.001}$ & $\underline{.904}\sd{.002}$ & $\mathbf{.794}\sd{.001}$ & $\mathbf{.818}\sd{.004}$ & $\mathbf{.785}\sd{.001}$ & $\underline{.980}\sd{.001}$ & $\underline{.760}\sd{.005}$ & $\mathbf{.926}\sd{.003}$ & $\mathbf{.842}\sd{.002}$ & $\mathbf{.894}\sd{.004}$ \\
\bottomrule
\end{tabular}
\end{table*}

\subsection{Sample Quality}\label{sec:supp_samples}

A low NLL does not guarantee good samples, since in high dimensions likelihood and sample quality can disagree~\citep{theis2016note}.
The previous subsection tested the learned densities; in this section, we check the quality of the samples drawn from them, since the aleatoric score relies on sampling (\cref{sec:uq:aleatoric}).
We test whether the samples fall in the regions occupied by the data, on CC1M, the distribution the densities are fitted to, with both backbones and five seeds.

\paragraph{Protocol.}
We split $5{,}000$ held-out CC1M pairs into a query set $\mathcal{D}_{\mathrm{query}}$ and a reference set $\mathcal{D}_{\mathrm{ref}}$ of $2{,}500$ pairs each, so that no conditional sample is compared with its own pair.
For every pair in $\mathcal{D}_{\mathrm{query}}$, we fix the text and sample an image from $p(\mathbf{e}_I\mid\mathbf{e}_T)$, and fix the image and sample a text from $p(\mathbf{e}_T\mid\mathbf{e}_I)$.
We also draw $2{,}500$ samples directly from each marginal.
The image and text embeddings of $\mathcal{D}_{\mathrm{ref}}$ serve as the reference for samples of the same modality, both marginal and conditional, since conditional samples pooled over $\mathcal{D}_{\mathrm{query}}$ follow the marginal distribution. Among the baselines, only REPVLM allows sampling from the data distribution, and only from its two marginal densities; the others place a distribution around a given embedding, which perturbs that embedding and does not generate new ones.

A well-learned distribution should generate samples with two properties: \emph{fidelity} (each sample falls in a region occupied by real embeddings) and \emph{diversity} (together, the samples spread over the population of real embeddings).
We measure both with one construction: around each reference embedding, we place a ball whose radius is the distance to its $k$-th nearest other reference embedding, with $k=3$~\citep{kynkaanniemi2019improved}.
Together, the balls approximate the region occupied by the data, and they are small where the data are dense and large where they are sparse.
Distances are Euclidean, which orders neighbors on the sphere as geodesic distance does.
\emph{Precision}~\citep{kynkaanniemi2019improved} measures fidelity: it is the share of samples that fall inside at least one ball.
\emph{Coverage}~\citep{naeem2020reliable} measures diversity: it is the share of balls that contain at least one sample.
For comparison, we add a calibration scale, in which the real embeddings of $\mathcal{D}_{\mathrm{query}}$ are scored with the same balls and the same number of samples (Data row of \cref{tab:supp_samples_fid}).
Even real embeddings stay below $1$: in precision because some fall outside every ball, and in coverage because some balls receive no sample.
These values are the reference for real data under this protocol, and we call them the data level.

\paragraph{Results.}
The conditional samples of GeoFlowVLM, which the aleatoric score uses (\cref{sec:uq:aleatoric}), are the closest to the data level: within $0.04$ of it in precision and coverage on both backbones, for images and texts alike (\cref{tab:supp_samples_fid}).
Its marginal samples lie further below this level, and those of REPVLM are closer to it by $0.01$ to $0.06$, although GeoFlowVLM attains the lower NLL (\cref{tab:supp_likelihood}).
This is the disagreement noted above; it does not affect the uncertainty scores, which never use marginal samples.

\begin{table}[t]
\centering
\footnotesize
\setlength{\tabcolsep}{3pt}
\newcommand{\sd}[1]{{}_{\textcolor{black!55}{\pm#1}}}
\caption{Precision and coverage ($\uparrow$) of samples against held-out CC1M embeddings, mean $\pm$ std over five seeds. Data: real held-out embeddings, scored in the same way. Bold: better marginal model.}
\label{tab:supp_samples_fid}
\begin{tabular}{@{}cllcccc@{}}
\toprule
 & & & \multicolumn{2}{c}{\textbf{Image}} & \multicolumn{2}{c}{\textbf{Text}} \\
\cmidrule(lr){4-5}\cmidrule(l){6-7}
 & \textbf{Samples} & \textbf{Model} & Prec. & Cov. & Prec. & Cov. \\
\midrule
 & Data & -- & $.85$ & $.88$ & $.81$ & $.86$ \\
 & Marginal & REPVLM & $\mathbf{.78}\sd{.01}$ & $\mathbf{.86}\sd{.01}$ & $\mathbf{.74}\sd{.01}$ & $\mathbf{.85}\sd{.01}$ \\

 & & GeoFlowVLM & $.74\sd{.01}$ & $.82\sd{.00}$ & $.73\sd{.01}$ & $.84\sd{.01}$ \\

\multirow{-4}{*}{\rotatebox[origin=c]{90}{CLIP}} & Conditional & GeoFlowVLM & $.81\sd{.01}$ & $.85\sd{.01}$ & $.81\sd{.01}$ & $.88\sd{.00}$ \\
\midrule
 & Data & -- & $.80$ & $.86$ & $.75$ & $.87$ \\
 & Marginal & REPVLM & $\mathbf{.77}\sd{.00}$ & $\mathbf{.87}\sd{.01}$ & $\mathbf{.75}\sd{.01}$ & $\mathbf{.90}\sd{.00}$ \\
 & & GeoFlowVLM & $.71\sd{.01}$ & $.81\sd{.01}$ & $.69\sd{.01}$ & $.85\sd{.01}$ \\
\multirow{-4}{*}{\rotatebox[origin=c]{90}{SigLIP}} & Conditional & GeoFlowVLM & $.78\sd{.01}$ & $.86\sd{.00}$ & $.74\sd{.01}$ & $.86\sd{.01}$ \\
\bottomrule
\end{tabular}
\end{table}

\subsection{Out-of-Distribution Detection}\label{sec:supp_ood}
We evaluate how well GeoFlowVLM separates images of the training distribution from out-of-distribution (OOD) images, which come from other distributions.
This tests whether epistemic uncertainty responds to distribution shift, complementing selective prediction, which evaluates uncertainty through VLM prediction errors.
Since this test concerns images only, the conditional terms play no role, so we follow REPVLM and use the negative log marginal image density, $-\log p(\mathbf{e}_I)$, as the OOD score for both methods.
The comparison therefore assesses how well each model's learned image density distinguishes in-distribution from OOD images.

\paragraph{Protocol.}
The \emph{in-distribution} (ID) set consists of $10{,}000$ held-out CC1M images from the validation split of Supp.~\ref{sec:supp_training}, on which no method is trained; the OOD sets are the $10{,}000$ images of each of the five classification datasets (Supp.~\ref{sec:supp_datasets}), all of which differ from CC1M, each tested against the same ID set.
ProbVLM scores an image by the variance of its predicted image embedding over $15$ MC-dropout passes~\citep{gal2016dropout}.
We report two standard OOD metrics~\citep{hendrycks2017baseline,liang2018enhancing,ming2022delving}.
The \emph{area under the ROC curve} (AUROC) is the probability that a randomly drawn OOD image receives a higher score than a randomly drawn ID image, which is $1$ for perfect separation and $0.5$ for a score that carries no information.
The \emph{false-positive rate at $95\%$ true-positive rate} (FPR95) is the fraction of ID images flagged when $95\%$ of the OOD images are flagged.
We report the mean and standard deviation over the five trained models of each method, with both backbones.

\paragraph{Results.}
REPVLM leads GeoFlowVLM in most columns of \cref{tab:supp_ood}, mainly with CLIP, although GeoFlowVLM attains the lower NLL for the image marginal (Supp.~\ref{sec:supp_likelihood}).
The two results are compatible: OOD detection depends on the contrast between the densities of ID and OOD images, and not on their level~\citep{nalisnick2019deep}.
The differences are however small: AUROC and FPR95 differ by at most $0.04$, with one exception, FPR95 on ImageNet-R with CLIP ($0.07$), and with SigLIP the two methods alternate, GeoFlowVLM leading in AUROC on ImageNet-R, Food-101 and ObjectNet.
GeoFlowVLM reaches this level although it learns all four densities with one model, whereas REPVLM learns only the two marginals.
ProbVLM has the lowest AUROC on every set with both backbones.
Separation is strong only on CIFAR-100 with CLIP; on the other sets, FPR95 is $0.75$ or higher for every method.

\subsection{Inference Cost}\label{sec:supp_scores}
\paragraph{Protocol.}
We measure the GPU wall-clock time per query of each score on a single GPU (Supp.~\ref{sec:supp_compute}), from precomputed embeddings, so the cost reflects only the model trained on top of the frozen encoders.
The cost per query does not depend on the seed, the backbone or the dataset: it is set by the embedding dimension, the architecture and the solver settings, which are the same for all of them.
In this experiment, we use the first seed with CLIP ViT-H/14, and compute the aleatoric score on $5{,}000$ MS-COCO T2I queries and the epistemic score on the $10{,}000$ ImageNet-1K images of the selective-prediction benchmark.
We also report $\rho$ and Acc@$90\%$ of the same runs as a reference for quality.
For GeoFlowVLM, we add a cheaper setting with half the samples ($M=8$ to $M=4$), half the function evaluations (NFE $20$ to $10$) and one Hutchinson probe ($L=4$ to $L=1$).

\paragraph{Results.}
At the settings of the main paper, the aleatoric score takes $6.6$ ms per query, against under $0.01$ ms for AsymVLM and ProbVLM, which need one adapter pass, and $2.2$ ms for GroVE (\cref{tab:supp_cost}).
The epistemic score takes $0.69$ ms, about three times the $0.24$ ms of REPVLM, the only baseline that also evaluates a flow density.
GeoFlowVLM is slower because its score needs three network passes per step, one for the marginal log-density and two for the conditional one, where REPVLM needs one; this extra cost pays for the conditional term, and with it the aleatoric score, from the same model.
Its image term alone costs $0.22$ ms, as much as REPVLM.
The cheaper setting brings the two scores to $1.6$ and $0.17$ ms, about four times faster, which makes the epistemic score faster than REPVLM and still more accurate (Acc@$90\%$ of $0.837$ against $0.822$).
Against the main setting, the ranking quality is almost unchanged (Acc@$90\%$ of $0.837$ against $0.836$, $\rho$ of $-0.988$ against $-1.000$), since both protocols depend only on the order in which the scores rank the queries, which small errors in the density values rarely change.
The settings of the main paper were chosen for accurate density values, which the studies of Supp.~\ref{sec:supp_likelihood} and Supp.~\ref{sec:supp_samples} require, and the fast setting is one empirical point, not a tuned optimum.

\begin{table}[t]
\centering
\footnotesize
\setlength{\tabcolsep}{4pt}
\caption{Inference time per query (ms) with CLIP. Shading indicates the settings used in the main paper.}
\label{tab:supp_cost}
\begin{subtable}{\linewidth}
\centering
\caption{Aleatoric score, MS-COCO T2I.}
\begin{tabular}{@{}llrr@{}}
\toprule
\textbf{Method} & \textbf{Settings} & \textbf{Time (ms)} & $\boldsymbol{\rho}\downarrow$ \\
\midrule
AsymVLM$_{\mathrm{vMF}}$ & -- & $0.0007$ & $-1.000$ \\
AsymVLM$_{\mathrm{PS}}$ & -- & $0.0007$ & $-.988$ \\
ProbVLM & -- & $0.0012$ & $-.988$ \\
GroVE & -- & $2.2$ & $+.212$ \\
\rowcolor{gray!15}
GeoFlowVLM & $M{=}8$, NFE $20$, $L{=}4$ & $6.6$ & $-1.000$ \\
GeoFlowVLM (fast) & $M{=}4$, NFE $10$, $L{=}1$ & $1.6$ & $-.988$ \\
\bottomrule
\end{tabular}
\end{subtable}

\medskip
\begin{subtable}{\linewidth}
\centering
\caption{Epistemic score, ImageNet-1K.}
\begin{tabular}{@{}llrr@{}}
\toprule
\textbf{Method} & \textbf{Settings} & \textbf{Time (ms)} & \textbf{Acc@90\%}$\uparrow$ \\
\midrule
ProbVLM & -- & $0.014$ & $.790$ \\
REPVLM & -- & $0.24$ & $.822$ \\
\rowcolor{gray!15}
GeoFlowVLM & NFE $20$, $L{=}4$ & $0.69$ & $.836$ \\
GeoFlowVLM (fast) & NFE $10$, $L{=}1$ & $0.17$ & $.837$ \\
GeoFlowVLM (OOD score) & NFE $20$, $L{=}4$ & $0.22$ & $.814$ \\
\bottomrule
\end{tabular}
\end{subtable}
\end{table}

\begin{figure*}[t]
\centering
\includegraphics[width=0.75\textwidth]{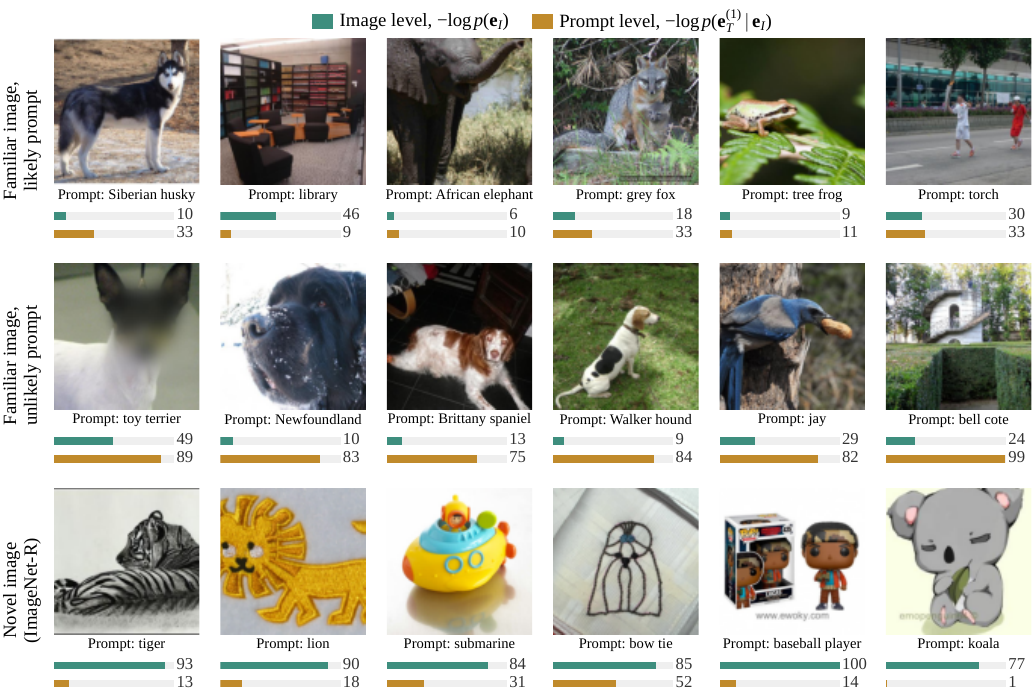}
\caption{Image and prompt levels of the epistemic score with CLIP ($0$: certain, $100$: uncertain) for six random images per group. The first two rows are from ImageNet-1K.}
\label{fig:supp_qualitative_epistemic}
\end{figure*}

\begin{figure*}[t]
\centering
\includegraphics[width=0.8\textwidth]{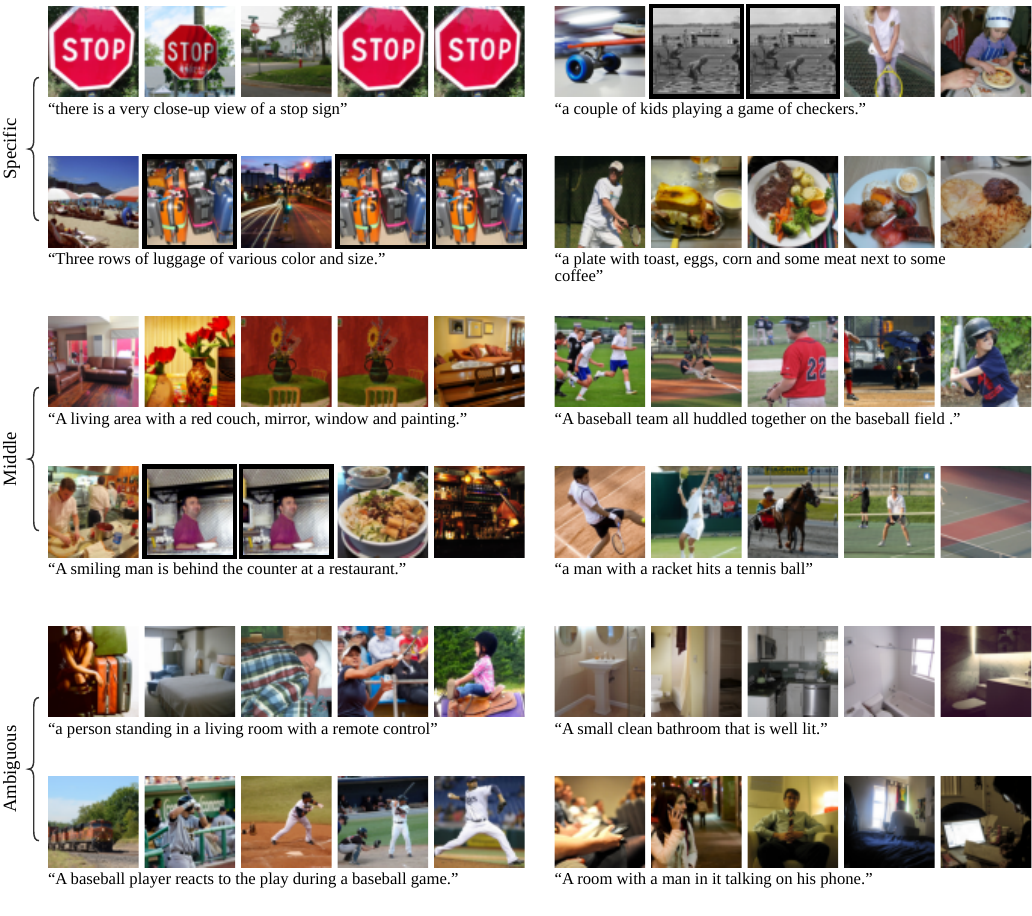}
\caption{Four typical captions at each of three levels of the aleatoric score, each with the MS-COCO images retrieved for five samples of $p(\mathbf{e}_I\mid\mathbf{e}_T)$. Black frame: the caption's own image. Captions are selected at the group median of the similarity between their five samples.}
\label{fig:supp_qualitative_aleatoric}
\end{figure*}


\subsection{Qualitative Analysis of the Epistemic Score}\label{sec:supp_qualitative_epistemic}

\Cref{eq:epi_score} defines the epistemic score as the sum of two terms, $u_{\mathrm{epistemic}}=-\log p(\mathbf{e}_I)-\log p(\mathbf{e}_T^{(1)}\mid\mathbf{e}_I)$: the image term, which asks whether the image is familiar to the training distribution, and the prompt term, which asks whether the predicted class is a likely pairing for this image.
We show on real images what each term responds to, in the three cases of \cref{sec:uq:epistemic}.

\paragraph{Protocol.}
We use ImageNet-1K for the two familiar cases and ImageNet-R for the novel case (Supp.~\ref{sec:supp_datasets}).
We score the images of both sets with each term, using CLIP, and express each score as a percentile level among the ImageNet-1K images, from $0$ (most certain) to $100$ (most uncertain).
We draw six images at random from each of three groups:
(i) \textit{familiar image, likely prompt}: both levels $<50$;
(ii) \textit{familiar image, unlikely prompt}: image level $<50$ and prompt level $>75$;
(iii) \textit{novel image}: image level $>75$.

\paragraph{Results.}
\Cref{fig:supp_qualitative_epistemic} shows six images of each group, one group per row.
In the first row (\textit{familiar image, likely prompt}), both the image and prompt term are low: these are ordinary photos of common animals and scenes, such as an elephant or a tree frog.
In the second row (\textit{familiar image, unlikely prompt}), the image term is still low, since the photos are just as ordinary, but the prompt term is high: the predicted class is either a less common fine-grained name, such as the dog breeds Walker hound and Brittany spaniel, or a poor fit, such as a bell cote for a hedge maze.
In the third row (\textit{novel image}), the image term is high, since these are renditions far from the training photos, such as sketches, embroideries, toys and cartoons.
The prompt term, which plays no role in the selection of this group, stays below $50$ for five of the six. This is consistent with the two terms responding to different properties of a prediction, which the analysis below tests across classes.

\paragraph{Prompt term and caption frequency.}
The prompt term is learned from the image--caption pairs of CC1M, so it should be low for the classes that commonly occur in the training captions, and high for the classes that rarely occur. The second row of \cref{fig:supp_qualitative_epistemic} shows this: five of the six predicted names never occur in the training captions.
To test this, we check how often the predicted ImageNet-1K classes appear in the training captions of CC1M.
For each class, we count the captions that contain its name or its plural, and we average the level of each term over the images that CLIP predicts as this class.
We then compute the Spearman correlation ($\rho$) between the count and the mean level, over the $928$ classes that CLIP predicts for at least five images, so that each mean rests on enough images. For the prompt term, we expect a negative correlation (a class that occurs often in the captions is a well-supported pairing, so the term is lower), and we find $\rho=-0.588$.
For the image term, we expect a correlation near zero (this term never sees the prompt), and we find a weak one, $\rho=+0.173$.
The count is only a proxy, since the score works on embeddings, so a prompt can be well supported even when its class name never occurs in the captions.


\subsection{Qualitative Analysis of the Aleatoric Score}\label{sec:supp_qualitative_aleatoric}
The aleatoric score compares the conditional $p(\mathbf{e}_I\mid\mathbf{e}_T)$ of images given a caption with the marginal $p(\mathbf{e}_I)$ (\cref{sec:uq:aleatoric}).
The marginal is the same for every caption, so what separates a specific caption from an ambiguous one is the conditional.
We therefore look at samples from $p(\mathbf{e}_I\mid\mathbf{e}_T)$: they should stay close together for a specific caption and spread out for an ambiguous one.

\paragraph{Protocol.}
We use text-to-image (T2I) retrieval on MS-COCO (Supp.~\ref{sec:supp_datasets}), where each of the $25{,}000$ test captions is a query for its image.
We compute the aleatoric score of every caption with CLIP and express it as a percentile level among all captions, from $0$ (most certain) to $100$ (most uncertain).
We then form three groups of $2{,}500$ captions each:
(i) \textit{specific}: level $<10$;
(ii) \textit{middle}: level $45$--$55$;
(iii) \textit{ambiguous}: level $>90$.

For every caption of the three groups, we draw five i.i.d.\ image embeddings from the conditional and match each one to the most similar of the $5{,}000$ test images by cosine similarity.
Each sample is matched on its own, so two different samples can return the same image: a repeat means that the samples lie close together, not that they are identical.
For each group, we report three rates over its captions:
\textit{repeated match}, the share of captions for which two or more samples return the same image;
\textit{own image matched}, the share for which a sample returns the caption's own image;
and \textit{CLIP Recall@1}, the share for which CLIP's retrieval ranks the caption's own image first.

\paragraph{Results.}
\Cref{fig:supp_qualitative_aleatoric} shows four typical captions of each group, each with the five test images retrieved for its five sampled embeddings.
For \textit{specific} captions, the five samples often return the same image, and often the caption's own image, as for the stop sign and the rows of luggage.
For \textit{ambiguous} captions, the samples return different images of the same scene type, such as different bathrooms or different moments of a baseball game.
\Cref{tab:supp_aleatoric_cases} confirms this over all captions of each group: from \textit{specific} to \textit{ambiguous} captions, the rate of repeated matches falls from $70.3\%$ to $21.0\%$, and the rate of own image matches from $37.9\%$ to $11.8\%$.
Since the samples are i.i.d., repeats occur only when the conditional is narrow, so these rates show a narrow conditional for specific captions and a wide one for ambiguous captions.
The wide conditional is also where CLIP's retrieval fails more often ($30.0\%$ Recall@1, against $59.8\%$ for specific captions and $48.5\%$ over all captions), since many images fit the caption about equally.
Some matches in every group are unrelated to their caption, because a sample can land where none of the $5{,}000$ gallery images is close; the rates, computed over all captions of a group, carry the comparison.

\begin{table}[H]
\centering
\footnotesize
\setlength{\tabcolsep}{7pt}
\caption{Rates over the $2{,}500$ captions of each group (MS-COCO, T2I). A match is the test image with the highest cosine similarity to a conditional sample.}
\label{tab:supp_aleatoric_cases}
\begin{tabular}{@{}lccc@{}}
\toprule
Group & Repeated match & Own image matched & CLIP Recall@1 \\
\midrule
Specific  & $70.3\%$ & $37.9\%$ & $59.8\%$ \\
Middle    & $43.1\%$ & $23.6\%$ & $51.2\%$ \\
Ambiguous & $21.0\%$ & $11.8\%$ & $30.0\%$ \\
\midrule
All captions & -- & -- & $48.5\%$ \\
\bottomrule
\end{tabular}
\end{table}

\end{document}